\documentclass[11pt]{article}

\usepackage[final]{acl}

\usepackage{times}
\usepackage{latexsym}

\usepackage[T1]{fontenc}
\usepackage[utf8]{inputenc}

\usepackage{microtype}

\usepackage{inconsolata}

\usepackage{graphicx}
\graphicspath{{figures/}}

\usepackage{listings}
\lstdefinestyle{mystyle}{
    backgroundcolor=\color{lightgray},   
    commentstyle=\color{green},
    keywordstyle=\color{magenta},
    numberstyle=\tiny\color{gray},
    stringstyle=\color{red},
    basicstyle=\ttfamily\footnotesize,
    breakatwhitespace=false,         
    breaklines=true,                 
    captionpos=b,                    
    keepspaces=true,                 
    numbers=left,                    
    numbersep=5pt,                  
    showspaces=false,                
    showstringspaces=false,
    showtabs=false,                  
    tabsize=2,
    moredelim=[is][\color{red}]{*}{*}, % Define a custom delimiter to color text in red
    moredelim=[is][\color{blue}]{\{}{\}}, % Define a custom delimiter to color text in blue
    moredelim=[is][\color{orange}]{\%}{\%}, % Define a custom delimiter to color text in orange
}
\usepackage[ruled, vlined]{algorithm2e}
\usepackage{multicol}
\usepackage{diagbox}
\usepackage{multirow}
\usepackage{xcolor}
\usepackage{subfiles}
\usepackage{amsmath}
\usepackage{graphicx}
\usepackage{amsthm}
\usepackage{pifont}
\usepackage{subcaption}
\usepackage{booktabs}       % professional-quality tables
\usepackage{placeins}
\usepackage{amsfonts}       % blackboard math symbols
\usepackage[table]{xcolor}         % colors

\newcommand{\tb}[1]{\textbf{#1}}
\newcommand{\mc}[1]{\mathcal{#1}}
\newcommand{\E}{\mathbb{E}}
\newcommand{\revised}[1]{#1}

\newcommand{\revise}[1]{#1}
\newcommand{\newrevise}[1]{#1}
\newcommand{\camerarevise}[1]{#1}
\newenvironment{camerarevision}[1]{#1}

\theoremstyle{plain}
\newtheorem{theorem}{Theorem}[section]
\newtheorem{proposition}[theorem]{Proposition}

\theoremstyle{definition}

\theoremstyle{remark}

\usepackage{tabularx}    % for auto-width tables
\newcolumntype{Y}{>{\centering\arraybackslash}X}
\usepackage{booktabs}    % for \toprule etc.
\newcommand{\cmark}{\textcolor{green!60!black}{\ding{51}}} % ✓ green
\newcommand{\xmark}{\textcolor{red}{\ding{55}}}           % ✗ red
\title{From Prompts to Trees: Effective LLM-Guided Tree Generation for Few-Shot Tabular Classification}

\author{
Yue Qiu$^{1*}$,
Zekang Du$^{1*}$,
Yiqun Diao$^{2}$,
Bingsheng He$^{2}$,
Qinbin Li$^{1\dagger}$ \\
$^{1}$School of Computer Science and Technology, Huazhong University of Science and Technology \\
$^{2}$School of Computing, National University of Singapore \\
\texttt{\{yueqiu, zekangdu, qinbin\}@hust.edu.cn},\\
\texttt{yiqun@comp.nus.edu.sg},
\texttt{dcsheb@nus.edu.sg}
}

\begin{document}
\maketitle
\begingroup
\renewcommand{\thefootnote}{}
\footnotetext{%
\begin{tabular}{@{}l@{}}
\(*\) Equal contribution.\\
\(\dagger\) Corresponding author.
\end{tabular}}
\addtocounter{footnote}{-1}
\endgroup
\begin{abstract}
While Large Language Models (LLMs) possess rich world knowledge and impressive generalization capabilities, their direct application to tabular data classification is hindered by high inference costs and limited interpretability. In contrast, decision trees are fast and transparent but often underperform in low-data regimes. In this work, we propose a novel framework that bridges these paradigms by distilling LLM knowledge into interpretable decision trees under a few-shot learning setting. Instead of directly prompting the LLM to generate full trees, which is often unstable and inefficient, we develop a three-stage paradigm that prompts the LLM to generate rules and organize the rules into a tree. Experiments on multiple real-world tabular datasets demonstrate that our method achieves superior accuracy and interpretability with significantly lower prompting overhead compared to existing baselines.
\end{abstract}

\section{Introduction}
Large Language Models (LLMs)~\cite{openai2023gpt4,touvron2023llama,liu2024deepseek,naveed2023comprehensive,bai2023qwen} have demonstrated remarkable capabilities in various domains, including code generation, planning, and complex problem-solving tasks. 
Despite these advancements, their potential remains significantly underexplored in the context of traditional tabular data classification~\cite{shwartz2022tabular,borisov2022deep}. Tabular data is still the widely used format of data in reality, and is ubiquitous in critical domains such as healthcare diagnostics, financial risk assessment, and customer behavior prediction.

Although LLMs inherently possess extensive knowledge, their direct application to tabular data classification faces two prominent challenges: (1) limited interpretability, making it difficult to provide intuitive justifications for classification decisions, and (2) inefficiency in inference, which restricts their applicability in resource-constrained and real-time scenarios. Decision trees~\cite{loh2011classification,chen2016xgboost,ke2017lightgbm,maimon2014data,quinlan1986induction,quinlan2014c4}, on the other hand, are widely adopted for tabular data classification due to their inherent interpretability, high computational efficiency, and ability to capture non-linear relationships between features. It has been a dominant model for tabular data classification for a long period. \revised{However, traditional tree training requires substantial labeled data to capture non-linear feature interactions. While LLMs can extract prior semantic knowledge from textified features to compensate for this data scarcity, combining them as a knowledge source and an efficient executor remains a challenge in practical few-shot settings.} 

This raises an important question: \emph{can we present the knowledge of LLMs by decision trees for tabular data classification in few-shot settings?} If successful, such a tree representation would simultaneously inherit the LLM’s rich prior knowledge while retaining the interpretability and fast inference of decision trees, thereby addressing the key limitations of directly applying LLMs to tabular classification. A straightforward approach is to prompt the LLM once to directly generate a tree to represent its knowledge. However, this yields unreliable results due to the inherent complexity and instability of long-context reasoning. Another method is to generate the reasoning process by prompting multiple times. For example, Tree-of-Thought~\cite{yao2023tree} approaches involve recursively prompting LLMs to generate the reasoning paths in a tree structure. Nevertheless, such methods suffer from substantial prompting overhead when applied to tabular data classification, rendering them impractical for wide adoption.

In this work, we propose \textbf{LLMT}, a modular and efficient framework that distills LLM knowledge into decision trees through a three-stage pipeline: (1) \emph{Rule Generation}, where the LLM produces structured, machine-parseable classification rules; (2) \emph{Tree Assembly}, where these rules are organized into a tree based on Gini impurity computed from a small labeled dataset; and (3) \emph{Leaf Refinement}, where the LLM corrects noisy leaf decisions. \camerarevise{Unlike prompting-based methods such as Chain-of-Thought (CoT) and Tree-of-Thought (ToT), which construct a reasoning chain or tree in a single pass or through recursive expansion, LLMT breaks the task into distinct stages. Our experiments show that this modular design improves both stability and efficiency.}

Our contributions are as follows:
\begin{itemize}
    \item We introduce a \textbf{novel three-stage distillation paradigm} that bridges the gap between expensive, opaque LLM reasoning and efficient, interpretable tree models.
    \item We provide a \textbf{formal generalization analysis} of LLMT under explicit structural assumptions, clarifying its statistical properties.
    \item We conduct \textbf{extensive experiments} on \camerarevise{11 datasets with 17 baselines}, reporting accuracy and efficiency across different settings.
    \item We show that LLMT achieves consistent and significant accuracy improvements, large token-cost savings, and interpretable outputs in the few-shot setting.
\end{itemize}

\begingroup
\renewcommand{\thefootnote}{}
\footnotetext{Code is available at \href{https://github.com/yueqiu0/LLMTree}{\nolinkurl{github.com/yueqiu0/LLMTree}}.}
\addtocounter{footnote}{-1}
\endgroup
\section{Background and Related Work}

\subsection{Decision Trees}
Decision trees~\citep{loh2011classification,maimon2014data,quinlan1986induction,quinlan2014c4} are among the most fundamental models in machine learning, widely used in tasks requiring interpretable decision-making. A decision tree forms a hierarchy of if--then rules: internal nodes test features, branches denote test outcomes, and leaves assign class labels. Trees are learned by recursively selecting splits that reduce impurity, such as Gini impurity~\citep{loh2011classification}.

\subsection{LLMs for Tabular Data}

Tabular data is one of the most prevalent formats in real-world applications and can be naturally converted into text by enumerating feature-value pairs. Recent work has explored LLMs for tabular prediction via fine-tuning, instruction tuning, synthetic data generation, or direct prompting~\citep{wang2023anypredict,herzig2020tapas,zhang2023generative,hegselmann2023tabllm,jaitly2023better,fang2024large,dinh2022lift,zhang2024towards}. TabLLM~\citep{hegselmann2023tabllm} fine-tunes T0~\citep{sanh2021multitask} with a few tabular examples, while InsightTab~\citep{insighttab2025} and SumBoost~\citep{sumboost2023} use LLMs directly as tabular classifiers with improved prediction pipelines. Although effective, these methods rely on LLMs during inference, leading to limited interpretability and high prediction cost.

Several recent works further investigate whether LLMs can generate rule-based or tree-like structures for tabular data~\citep{knauer2024oh,carrasco2025zero,Xiongetal2024}. Direct-ZSDT~\citep{knauer2024oh} and Step-ZSDT~\citep{carrasco2025zero} build decision trees zero-shot but cannot use labeled examples to guide construction. GPTree~\citep{Xiongetal2024} uses LLMs to pre-generate rule candidates from raw training samples, making it more suitable for data-rich settings but less robust in few-shot regimes where limited or noisy examples may bias rule generation.

Other related methods combine LLMs with downstream estimators. FeatLLM~\citep{Hanetal2024} uses LLMs for feature engineering before training a linear model, while DeLTa~\citep{Yeetal2025} uses LLM-refined rules to guide residual error correction in tree ensembles. In contrast, our method aims to generate a standalone decision tree: it first extracts structured rules from the LLM's prior knowledge and then organizes them statistically using few-shot examples, achieving interpretability, low inference cost, and data-aware tree construction.

\begin{table*}[!]
\centering
{\scriptsize
\setlength{\tabcolsep}{1.2pt}
\renewcommand{\arraystretch}{1.05}
\newcommand{\tableonehead}[1]{{\fontsize{6.5}{7}\selectfont\bfseries #1}}
% Keep method names at their native size and distribute whitespace evenly.
\begin{tabular*}{\textwidth}{@{}>{\raggedright\arraybackslash}p{3.35cm}@{\extracolsep{\fill}}*{12}{c}@{}}
\toprule
\tableonehead{Features}
& \tableonehead{LLMT}
& \revise{\tableonehead{Direct-ZSDT}}
& \revise{\tableonehead{Step-ZSDT}}
& \tableonehead{CoT-Tree}
& \tableonehead{ToT-Tree}
& \tableonehead{IO-Tree}
& \tableonehead{TabLLM}
& \revise{\tableonehead{InsightTab}}
& \revise{\tableonehead{SumBoost}}
& \tableonehead{FeatLLM}
& \revised{\tableonehead{DeLTa}}
& \revised{\tableonehead{GPTree}} \\
\midrule
\mbox{Few-shot ready (no finetune)}      & \cmark & \xmark & \xmark & \cmark & \cmark & \cmark & \xmark & \cmark & \cmark & \cmark & \cmark & \cmark \\
\mbox{Token-efficient build}      & \cmark & \cmark & \xmark & \cmark & \xmark & \cmark & \xmark & \xmark & \xmark & \xmark & \xmark & \xmark \\
\mbox{Tree generation}   & \cmark & \cmark & \cmark & \cmark & \cmark & \cmark & \xmark & \xmark & \xmark & \xmark & \xmark & \cmark \\
\mbox{No LLM at inference}               & \cmark & \cmark & \cmark & \cmark & \cmark & \cmark & \xmark & \xmark & \xmark & \cmark & \cmark & \cmark \\
\mbox{Rule-first (vs path-first)}        & \cmark & \xmark & \xmark & \xmark & \xmark & \xmark & \xmark & \xmark & \xmark & \cmark & \xmark & \cmark \\
\mbox{\revised{Sample-level privacy preservation}} & \cmark & \cmark & \cmark & \xmark & \xmark & \xmark & \xmark & \xmark & \xmark & \xmark & \cmark & \xmark \\
\mbox{Theoretical generalization bound}  & \cmark & \xmark & \xmark & \xmark & \xmark & \xmark & \xmark & \xmark & \xmark & \xmark & \xmark & \xmark \\
\bottomrule
\end{tabular*}
}
\caption{Comparison of our LLM-guided tree method (\textbf{LLMT}) with other related studies}
\label{tab:llmt_comparison}
\end{table*}

\subsection{Prompting with Structured Reasoning}
Few-shot prompting is a widely adopted strategy to elicit reasoning capabilities from LLMs using only a handful of labeled examples. Prior work~\cite{wei2022chain,yao2023tree,ning2024skeletonofthought,besta2024graph} has shown that making the intermediate reasoning process explicit in the prompt significantly enhances LLM performance. Chain-of-Thought (CoT) prompting~\citep{wei2022chain} introduces the idea of decomposing complex tasks into a sequence of intermediate reasoning steps, leading to better accuracy.

Building on this, Tree-of-Thoughts (ToT)~\citep{yao2023tree} generalizes CoT by considering multiple branching reasoning paths, enabling more robust and explorative problem-solving. Skeleton-of-Thought (SoT)~\citep{ning2024skeletonofthought} further refines this approach by prompting the LLM to first generate a high-level reasoning structure and then elaborate on each component in parallel.

\paragraph{Positioning.}
Existing approaches fall into three categories: (i) using LLMs directly as in-context tabular classifiers~\citep{hegselmann2023tabllm,insighttab2025,sumboost2023}; (ii) using LLMs to generate features, rules, or residual corrections for downstream models~\citep{Hanetal2024,Yeetal2025}; and (iii) prompting LLMs to generate reasoning paths or tree-like structures~\citep{wei2022chain,yao2023tree,Xiongetal2024}. 
LLMT differs from these by separating knowledge extraction from tree construction: it extracts structured rules from feature semantics, assembles them statistically with few-shot labeled examples, and then performs leaf refinement.
Unlike data-driven rule-generation methods such as GPTree~\citep{Xiongetal2024}, LLMT does not rely on many raw training samples to induce rules, making it better suited to few-shot settings. 
\newrevise{Moreover, LLMT preserves sample-level privacy by deriving rules solely from feature semantics without transmitting raw data values to the LLM.}
Our design avoids using the LLM as the final classifier and yields an explicit decision tree with interpretability and low inference cost. 
Table~\ref{tab:llmt_comparison} summarizes the comparison.

% --- Preamble (add if not already present) ---
% \usepackage{pifont}      % for \ding symbols
% \usepackage{tabularx}    % for auto-width tables
% \usepackage{booktabs}    % for \toprule etc.
% \newcommand{\cmark}{\ding{51}}  % ✓
% \newcommand{\xmark}{\ding{55}}  % ✗
% \newcommand{\pmark}{\(\triangle\)} % partial/depends

% --- Preamble (add if not already present) ---

% --- Preamble additions ---
% \usepackage{pifont}      % for \ding symbols
% \usepackage{xcolor}      % for text colors
% \usepackage{tabularx}    % flexible column widths
% \usepackage{booktabs}    % nice horizontal rules
% \newcommand{\cmark}{\textcolor{green!60!black}{\ding{51}}} % ✓ green
% \newcommand{\xmark}{\textcolor{red}{\ding{55}}}           % ✗ red
% \newcommand{\pmark}{\(\triangle\)}                        % partial/depends

\section{Motivation}

We start by conducting experiments to gain insights into the performance and capabilities of LLMs in the context of tabular data classification. Our observations highlight both the strengths and limitations of LLMs, motivating our proposed approach.

\begin{figure}
    \centering
    \includegraphics[width=1\linewidth]{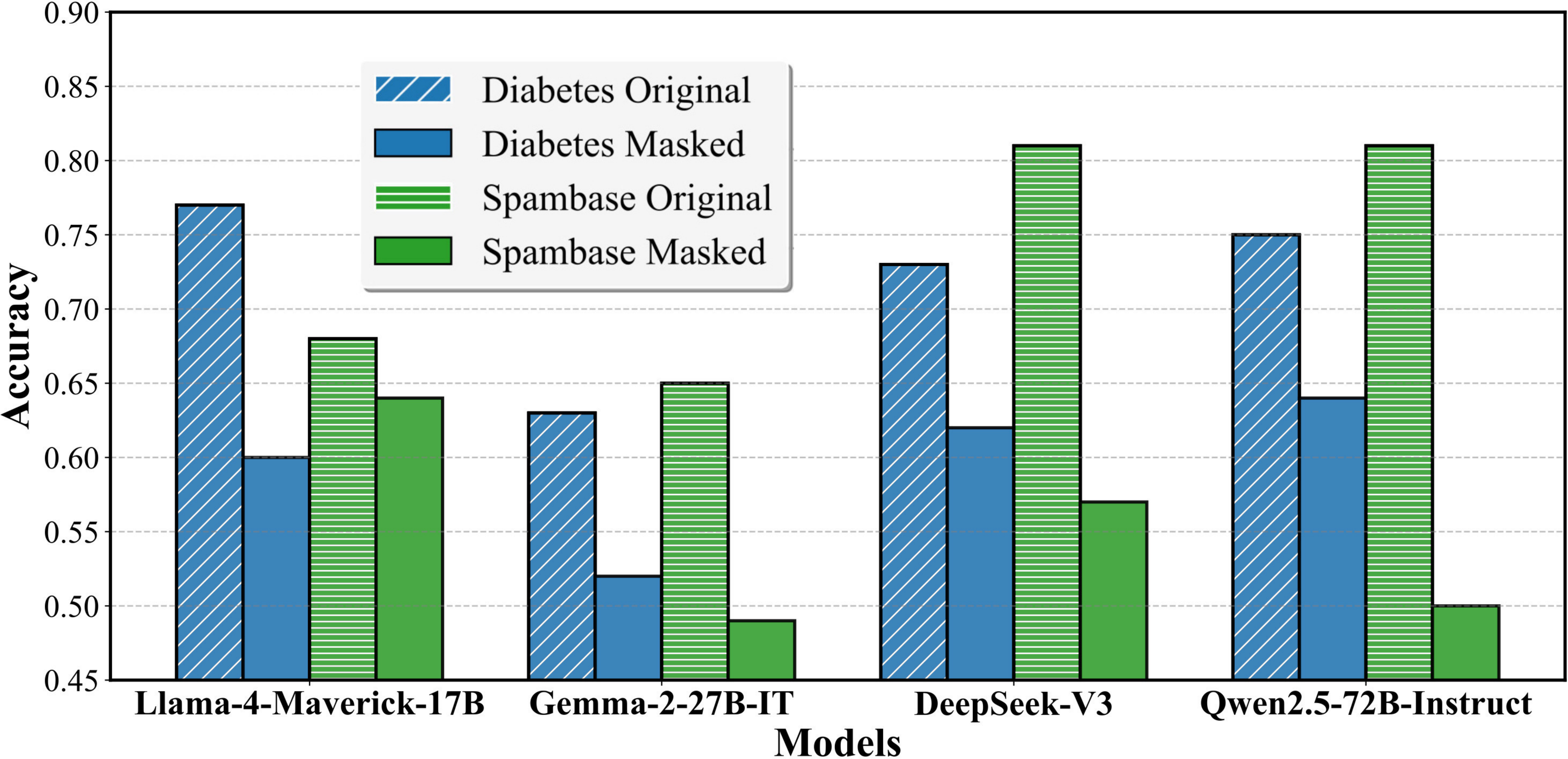}
    \caption{The test accuracy of LLMs with or without the metadata about the dataset.}
    \label{fig:mot1}
\end{figure}

\begin{figure}[t]
    \centering
    \begin{subfigure}[t]{0.48\linewidth}
        \centering
        \includegraphics[width=\linewidth]{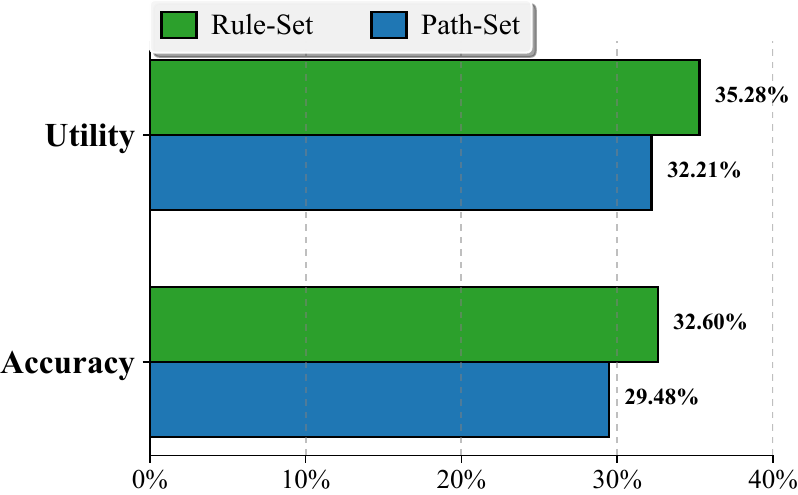}
        \caption{Diabetes}
        \label{fig:rule-path-spambase}
    \end{subfigure}
    \hfill
    \begin{subfigure}[t]{0.48\linewidth}
        \centering
        \includegraphics[width=\linewidth]{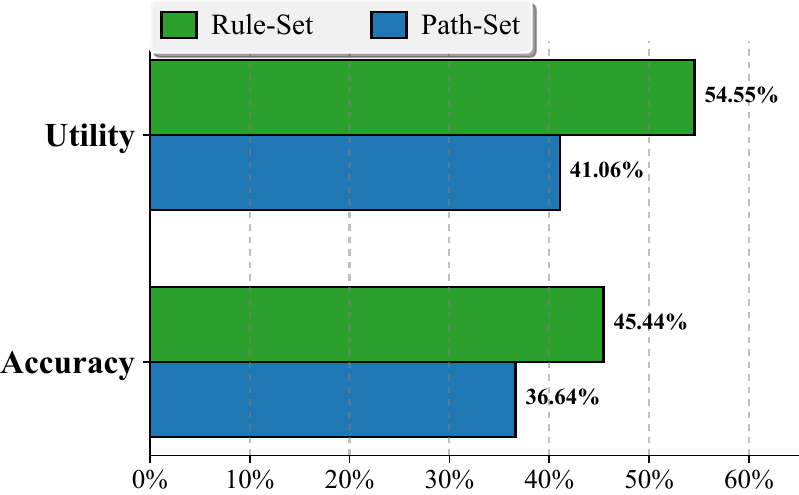}
        \caption{Spambase}
        \label{fig:rule-path-diabetes}
    \end{subfigure}
    \caption{\camerarevise{Comparison of Utility and Accuracy between generated rule sets and path sets}}
    \label{fig:mot2}
\end{figure}

\noindent \textit{
\colorbox{gray!20}{
\parbox{0.96\linewidth}{
\textbf{Observation 1: LLMs have rich background knowledge for tabular data classification.}
}}}

We assess the predictive performance of four LLMs on public tabular datasets under two scenarios: 1) \textit{original}, where we convert the tabular data into text and explicitly provide feature names and task descriptions in the query prompts; 2) \textit{masked}, where feature names and task descriptions are omitted \revise{(see Appendix~\ref{app:exp_settings} for details.)} The experimental results on two public tabular datasets \revise{Diabetes~\cite{diabetes_34} and Spambase~\cite{spambase_94}}, summarized in Figure~\ref{fig:mot1} under zero-shot prompting conditions, clearly show that the prediction accuracy of LLMs substantially deteriorates when metadata is removed. \camerarevise{To minimize the potential influence of pretraining memorization, we extend this comparison to six newer datasets, with complete results reported in Table~\ref{tab:newer_masking} of Appendix~\ref{app:obs1_scenarios}.}

\noindent \textit{
\colorbox{gray!20}{
\parbox{0.96\linewidth}{
\textbf{Observation 2: Compared to generating long reasoning paths, LLMs are better at producing discrete rules for tabular data classification.}
}}}

Recent studies~\cite{Hanetal2024} suggest that LLMs excel at feature engineering. To further investigate this capability, we conduct a comparative analysis of LLM performance in generating long reasoning paths versus discrete rules for tabular data classification. \revise{Reasoning paths can be seen as a series of rule-checks that traverse from the root to a leaf node in the decision tree.} \camerarevise{For a controlled comparison, we prompt LLMs to generate reasoning-path sets and discrete-rule sets containing the same total number of atomic feature conditions.} We then \camerarevise{evaluate the generated sets using Set Utility and Set Accuracy across eight datasets} \revise{(see Appendix~\ref{app:obs2_exp} for detailed settings and complete results).} The representative results, shown in Figure~\ref{fig:mot2}, reveal that rule generation achieves higher \camerarevise{Set Utility and Set Accuracy} than path generation. These findings indicate that while LLMs struggle to produce coherent and complete hierarchical reasoning paths in few-shot settings, they are particularly effective at generating accurate and interpretable individual classification rules.

These findings motivate our design choice: rather than asking the LLM to produce full reasoning paths, which are often unstable, lengthy, and error-prone, we instead focus on \emph{discrete rule generation and organization}.

\section{Methodology}

\subsection{Problem Statement}
We consider the problem of distilling the knowledge embedded in an LLM into an interpretable and efficient decision tree for tabular data classification. Formally, let $\mathcal{D} = \{(x_i, y_i)\}_{i=1}^{n}$ denote a tabular dataset with $n$ samples, where each input $x_i \in \mathbb{R}^m$ is a feature vector of $m$ attributes and $y_i\in\{1,2,\dots,K\}$ is a categorical label over $K$ classes. The goal is to learn a decision tree $T$ such that $T(x)$ approximates the classification capability of the LLM while maintaining interpretability and low computational overhead. 

Crucially, we assume access to only a small number of labeled samples (few-shot setting) and the ability to query an LLM via prompting. This constraint reflects practical scenarios where labeled tabular data is scarce but the model’s background knowledge is rich.

\subsection{Framework Overview}
To tackle the above challenge, we propose a novel three-stage distillation framework named \tb{LLMT} that efficiently extracts and structures the LLM's latent decision-making knowledge into a decision tree. Instead of prompting the LLM to generate the reasoning paths as a tree, we ask it to generate discrete rules and then organize them as a tree in a statistical way. As illustrated in Figure~\ref{fig:llmt}, our framework comprises the following components: (1) \textbf{Rule Generation}: We prompt the LLM with carefully designed instructions to generate discrete classification rules, along with a self-assessed confidence score for each rule. (2) \textbf{Tree Assembly}: We systematically organize the induced rules into a tree structure by selecting high-quality rules in a layer-wise manner using Gini-based splitting criteria on the small labeled dataset. (3) \textbf{Leaf Refinement}: We prompt the LLM to analyze each rule path, evaluate whether rules are satisfied, and assign confidence scores to possible labels, creating a probability distribution that quantifies prediction certainty. \revise{Please refer to Appendix~\ref{app:example} for an example of the execution process.}

\noindent\textbf{Implementation Details.} 
We adopt a level-wise \emph{heap indexing} scheme: the root is indexed as 1, its left and right children as 2 and 3, and so on in breadth-first order. At each node, we consider rules not yet used along the current path. Candidate rules are grouped by their confidence scores into \revised{multiple} bins, processed in descending order. Within each group, ties are broken by highest Gini gain and lowest index to ensure determinism. The process terminates when either the maximum tree depth is reached or no remaining rule yields a positive gain.

\begin{figure*}[t]
    \centering
    \includegraphics[width=0.95\linewidth]{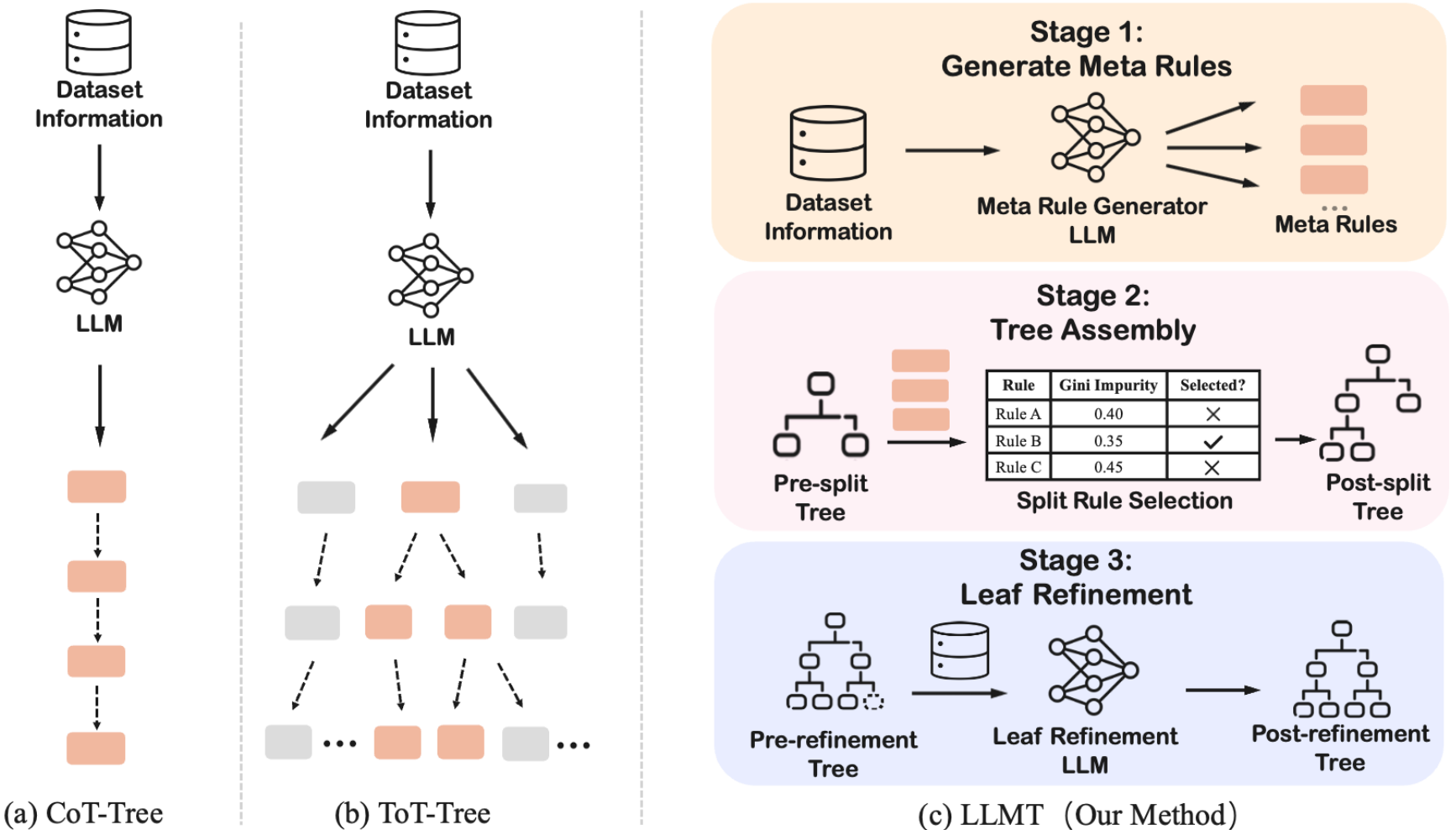}
    \caption{Comparison between our method and prompting-based methods for generating reasoning trees. Instead of (1) asking the LLM to directly generate long reasoning paths (CoT-Tree), or (2) asking the LLM to generate the next thought based on a long reasoning path, our method prompts the LLM to generate discrete rules, organizes them into a tree using statistical methods, and then uses the LLM for final refinement. }
    \label{fig:llmt}
\end{figure*}

\subsection{Rule Generation}

In the first stage, we prompt the LLM to generate a structured set of candidate rules. Each rule follows the grammar:
\revise{\[
 f_j \ \texttt{op}\ \theta \ ,
\]}
where \revise{$\texttt{op} \in \{<, \geq, =, \neq\}$}, $\theta$ is a numerical threshold or categorical value, \revise{and $f_j$ represents the $j$-th feature in the dataset (e.g., age $\geq$ 20).} Categorical features are encoded using one-vs-rest splits.

% The LLM receives (i) a \emph{system prompt} specifying the grammar and output format, and (ii) a \emph{dataset prompt} with feature names,types, and a few labeled examples. 

\revised{The LLM receives (i) a \emph{system prompt} detailing task constraints and rules, and (ii) a \emph{dataset prompt} containing feature names, types, and dataset-agnostic rule set examples for formatting.}

The model outputs rules with self-assessed confidence scores in \revise{$[1,10]$}.

\subsection{Tree Assembly}

The objective of this stage is to construct a decision tree $T$ using the rule set $\mathcal{R}$ and few-shot examples \revise{$\mathcal{D}_{\text{train}}$}. Our idea is to put the important rules on the top of the tree so that they can be shared in more reasoning paths, while using the knowledge in few-shot examples to adjust the node positions.

\noindent\textbf{Rule Grouping and Selection.} 
The process proceeds in a level-wise fashion through the tree. For each node at depth level $l$, we exclude rules already used on the path $R_{\mathrm{path}}$ to that node. We then group the remaining candidates by their confidence scores \revise{ (rules with the same confidence value form one group)}. Starting with the highest-confidence group, we select the optimal rule $R_k$ by minimizing the total Gini impurity of the resulting splits. Specifically, suppose the current tree node corresponds to a subset \revise{$\mc{D}_p \subseteq \mathcal{D}_{\text{train}}$} of samples. Then, we aim to find a rule $R_k$ as the node $T[i]$ such that
\begin{equation}
\label{eq:rule_selection}
    \arg\min_{R_k} \frac{|\mc{D}_l^k|}{|\mc{D}_p|}I_G(\mc{D}_l^k)+\frac{|\mc{D}_r^k|}{|\mc{D}_p|}I_G(\mc{D}_r^k),
\end{equation}
where $\mc{D}_l^k$ is the left child set \revised{split} by $R_k$ and $\mc{D}_r^k$ is the right child set \revised{split} by $R_k$. Intuitively, the following proposition holds:

\begin{proposition}[Expected Impurity Reduction]
Let \(r^*\in\mathcal{R}\) be the selected rule at node \(t\), splitting dataset \revise{\(\mc{D}_p \subseteq \mathcal{D}_{\text{train}}\)} into \(\mc{D}_l,\mc{D}_r\). Then, the expected impurity after splitting satisfies:
\[
\E[\mathrm{Gini}(\mc{D}_l,\mc{D}_r)] \leq \mathrm{Gini}(\mc{D}_p),
\]
with equality if and only if the split yields no information gain.
\end{proposition}

\subsection{Leaf Refinement}
While the internal structure of the tree is built from data-driven rule selection, the final classification decisions at leaf nodes may still be noisy or misaligned with LLM’s original knowledge. To address this, we introduce a refinement step that evaluates and corrects the leaf nodes if necessary.

For every path $p$ from the root to a leaf in $T$, we construct a natural language query describing the rule chain and prompt the LLM to evaluate its logical soundness. The LLM returns confidence scores $\mathbf{C_p}$ (a probability distribution over all possible labels). We identify the label $y^*$ with the highest confidence score $c^*$. If $c^* > \tau$ (a predefined threshold) and $y^*$ differs from the current leaf label, we adjust the classification at that leaf node to $y^*$ to improve consistency and generalization. As we show in Section~\ref{sec:ablation}, our leaf refinement process usually improves the accuracy of the tree. Using the PAC theory~\cite{blumer1989learnability}, we have the following theorem for the learned tree:

\begin{theorem}[Generalization Error Bound]
\label{thm:gen}

Consider a decision tree \(T\) constructed using LLMT of maximum depth \(d\) from a finite rule set \(\mathcal{R}\) with size \(K\). Given \(n\) training samples, for any \(\delta>0\), with probability at least \(1-\delta\), the generalization error \(R(T)\) satisfies:
\par \nobreak
\begingroup\small
\setlength{\abovedisplayskip}{2pt plus 1pt minus 2pt}
\setlength{\abovedisplayshortskip}{0pt plus 1pt}
\setlength{\belowdisplayskip}{8pt plus 2pt minus 2pt}
\setlength{\belowdisplayshortskip}{5pt plus 2pt minus 2pt}
\begin{align}
R(T) &\;\le\; \hat R(T)
\;+\; \sqrt{\frac{2\Bigl(\log P(K,2^d-1)+\log \tfrac{2}{\delta}\Bigr)}{n}}\\
&\;\le\; \hat R(T) \;+\;
\sqrt{\frac{2\Bigl((2^d-1)\log K+\log \tfrac{2}{\delta}\Bigr)}{n}},
\label{eq:genbound}
\end{align}
\endgroup\nobreak
where $P(K,2^d-1)=K!/(K-(2^d-1))!$.
\end{theorem}

The proof is available at Appendix~\ref{app:proof}. \revised{While this PAC-style bound can be numerically loose in the extreme few-shot regime ($n$ is small), it provides crucial structural guidance. The bound explicitly shows that tree depth $d$ is the dominant complexity driver (via the $2^d$ dependence), justifying our shallow tree design. Moreover, by distilling the LLM into a finite rule set $K$, LLMT induces a finite, statistically analyzable hypothesis space, distinguishing it from mathematically intractable black-box prompting pipelines.} As shown in our experiments (Appendix G.3), LLMT consistently outperforms baseline methods across varying tree depths.

\begin{camerarevision}
\subsection{Ensemble Extension}
For high-dimensional datasets, we optionally extend LLMT into an ensemble termed LLMT Forest, since ensembles can improve model capacity while reducing the overfitting risk of a single deep tree~\cite{breiman2001random,chen2016xgboost}. To construct LLMT Forest, we randomly permute and partition the features into disjoint, approximately equal-sized subsets, train one LLMT tree on each subset using the same few-shot samples, and aggregate their predictions by majority vote. 
\end{camerarevision}

\section{Evaluation}
\label{sec:exp}
We report representative results in the main paper. In the Appendix, we also include results on more baselines, more datasets and different tree depths.

\subsection{Experimental Setup}

\paragraph{Datasets}  
We evaluate our method on \camerarevise{11} widely used public tabular datasets including \textbf{Nursery}~\cite{nursery_76},  \textbf{Diabetes}~\cite{diabetes_34},
\textbf{Spambase}~\cite{spambase_94},
\textbf{Abalone}~\cite{abalone_94},
\textbf{Blood}~\cite{blood_08},
\textbf{Iris}~\cite{iris_36},
\revise{
\textbf{Breast}~\cite{breast_resistin_2018},
\textbf{Glioma}~\cite{glioma_voting_2022}, }
\camerarevise{\textbf{Communities}~\citep{redmond2002communities}, \textbf{Ecom}~\citep{tabred2024}, and \textbf{Myocardial}~\citep{golovenkin2020myocardial}. }
Among them, we select Diabetes, Spambase, and Nursery as the major datasets for evaluation across various settings. These datasets span diverse domains and vary widely in sample size and feature dimensionality, providing a comprehensive testbed.

\paragraph{Baselines}
\camerarevise{We compare LLMT against the following baselines, organized by methodological category:
(1) \emph{zero-shot decision-tree methods}: \textbf{DirectZSDT}~\cite{knauer2024oh}, \textbf{StepZSDT}~\cite{carrasco2025zero};
(2) \emph{conventional machine-learning methods}: \textbf{LogReg} (logistic regression; \citealp{cox1958regression}), \textbf{SVM} (support vector machine; \citealp{cortes1995support}), and \textbf{CART}~\cite{loh2011classification};
(3) \emph{tree-ensemble methods}: \textbf{Random Forest}~\cite{breiman2001random}, \textbf{XGBoost}~\cite{chen2016xgboost}, \textbf{CatBoost}~\cite{prokhorenkova2018catboost}, and \textbf{LightGBM}~\cite{ke2017lightgbm};
(4) \emph{LLM prompting methods}: \textbf{IO-Tree}, \textbf{CoT-Tree}~\cite{wei2022chain}, and \textbf{ToT-Tree}~\cite{yao2023tree};
(5) \emph{pretrained black-box predictors}: \textbf{TabPFN}~\citep{hollmann2025tabpfn} and \textbf{TabLLM}~\citep{hegselmann2023tabllm}; and
(6) \emph{recent LLM-assisted tabular learning methods}: \textbf{FeatLLM}~\cite{Hanetal2024}, \textbf{DeLTa}~\cite{Yeetal2025}, and \textbf{GPTree}~\cite{Xiongetal2024}.
Due to space constraints, additional comparisons involving Random Forest, CatBoost, TabPFN, and TabLLM are reported in Appendix~\ref{app:additional_results}.}

\paragraph{Settings} Experiments are run on a Linux server with 4$\times$ Intel Xeon Gold 5117 CPUs and 4$\times$ Nvidia Tesla V100 GPUs. We use Qwen2.5-72B-Instruct via TogetherAI~\cite{togetherai} by default, and provide all prompts in the Appendix. We set the maximum tree depth to 3 or 4 depending on the dataset, and \revise{set the number of rules as $K=\max(10,2^d-1)$, where $d$ is the maximum depth.} The leaf-refinement threshold is $\tau=0.7$ for all datasets. We repeat each experiment for 10 trials with randomly sampled training sets and a fixed test set of 100 samples, and report mean $\pm$ standard deviation. Additional details are in Appendix~\ref{app:exp_settings}.

\begin{table*}[!t]
\centering
\fontsize{7.4}{8.2}\selectfont
\setlength{\tabcolsep}{1.75pt}
\renewcommand{\arraystretch}{1.2}
\begin{tabular*}{\textwidth}{@{}l@{\extracolsep{\fill}}*{14}{c}@{}}
\toprule
\textbf{Dataset} 
& \textbf{DirectZSDT}
& \revised{\textbf{StepZSDT}}
& \revised{\textbf{LogReg}}
& \revised{\textbf{SVM}}
& \textbf{CART}
& \textbf{XGBoost}
& \revised{\textbf{LightGBM}}
& \textbf{IO-Tree}
& \textbf{CoT-Tree}
& \textbf{ToT-Tree}
& \revised{\textbf{FeatLLM}}
& \revised{\textbf{DeLTa}}
& \revised{\textbf{GPTree}}
& \textbf{LLMT} \\
\midrule

Nursery   
& 0.485
& \revised{0.450}
& \revised{0.558}
& \revised{0.564}
& 0.602
& 0.434
& \revised{0.578}
& 0.439
& 0.461
& 0.353
& \revised{0.571}
& \revised{0.448}
& \revised{0.479}
& \tb{0.762} \\

Diabetes  
& 0.729
& \revised{0.641}
& \revised{0.618}
& \revised{0.629}
& 0.546
& 0.690
& \revised{0.543}
& 0.539
& 0.701
& 0.561
& \revised{0.700}
& \revised{0.541}
& \revised{0.616}
& \tb{0.764} \\

Spambase   
& 0.645
& \revised{0.732}
& \revised{0.790}
& \revised{0.783}
& 0.723
& 0.600
& \revised{0.733}
& 0.723
& 0.739
& 0.727
& \revised{0.802}
& \revised{0.747}
& \revised{0.661}
& \tb{0.811} \\

Abalone 
& 0.550
& \revised{0.526}
& \revised{0.684}
& \revised{0.672}
& 0.689
& 0.670
& \revised{0.658}
& 0.647
& 0.694
& 0.570
& \revised{\tb{0.721}}
& \revised{0.656}
& \revised{0.667}
& \underline{0.694} \\

Blood 
& 0.597
& \revised{0.371}
& \revised{0.604}
& \revised{0.635}
& 0.565
& 0.605
& \revised{0.508}
& 0.585
& 0.500
& 0.551
& \revised{0.642}
& \revised{0.550}
& \revised{0.572}
& \tb{0.675} \\

Iris 
& \tb{0.946}
& \revised{0.547}
& \revised{0.793}
& \revised{0.756}
& 0.882
& 0.739
& \revised{0.685}
& 0.870
& 0.903
& 0.867
& \revised{0.903}
& \revised{0.823}
& \revised{0.814}
& \underline{0.940} \\

\revise{Breast} 
& \revise{0.554}
& \revised{0.500}
& \revised{0.597}
& \revised{0.607}
& \revise{0.500}
& \revise{0.400}
& \revised{0.513}
& \revise{0.543}
& \revise{0.473}
& \revise{0.593}
& \revised{0.544}
& \revised{0.529}
& \revised{0.523}
& \tb{0.617} \\

\revise{Glioma} 
& \revise{0.504}
& \revised{0.743}
& \revised{\tb{0.760}}
& \revised{0.755}
& \revise{0.592}
& \revise{0.560}
& \revised{0.680}
& \revise{0.637}
& \revise{0.598}
& \revise{0.652}
& \revised{0.724}
& \revised{0.643}
& \revised{0.678}
& \tb{0.760} \\

\midrule
\textbf{Average}
& 0.626
& \revised{0.564}
& \revised{0.676}
& \revised{0.675}
& 0.637
& 0.587
& \revised{0.612}
& 0.623
& 0.634
& 0.609
& \revised{0.701}
& \revised{0.617}
& \revised{0.626}
& \tb{0.753} \\

\bottomrule

\end{tabular*}
\caption{Accuracy comparison of different methods when the number of training examples per class is two. \camerarevise{Best performances are bolded, and our method's performances, when second-best, are underlined.}}
\label{tab:dataset_method_results}
\end{table*}

\begin{table*}[!t]
\centering
\small
\setlength{\tabcolsep}{3.5pt}
\renewcommand{\arraystretch}{1.1}
\resizebox{\textwidth}{!}{%
\begin{tabular}{l ccccc ccccc ccccc}
\toprule
\textbf{Datasets} 
& \multicolumn{5}{c}{\textbf{Nursery}} 
& \multicolumn{5}{c}{\textbf{Diabetes}} 
& \multicolumn{5}{c}{\textbf{Spambase}} \\
\cmidrule(lr){2-6} \cmidrule(lr){7-11} \cmidrule(lr){12-16}

\textbf{\#Shots} 
& 3 & 6 & 12 & 24 & 48 
& 2 & 4 & 8 & 16 & 32 
& 2 & 4 & 8 & 16 & 32 \\
\midrule

DirectZSDT 
& 0.485 & 0.485 & 0.485 & 0.485 & 0.485
& \textbf{0.729} & 0.729 & 0.729 & 0.729 & 0.729
& 0.645 & 0.645 & 0.645 & 0.645 & 0.645 \\

StepZSDT 
& 0.450 & 0.450 & 0.450 & 0.450 & 0.450
& 0.641 & 0.641 & 0.641 & 0.641 & 0.641
& 0.732 & 0.732 & 0.732 & 0.732 & 0.732 \\

\revised{LogReg}
& \revised{0.505} & \revised{0.558} & \revised{0.670} & \revised{0.778} & \revised{\textbf{0.855}}
& \revised{0.542} & \revised{0.618} & \revised{0.674} & \revised{0.730} & \revised{0.748}
& \revised{0.685} & \revised{0.790} & \revised{0.785} & \revised{0.782} & \revised{0.835}\\

\revised{SVM}
& \revised{0.500} & \revised{0.564} & \revised{0.661} & \revised{0.781} & \revised{0.833}
& \revised{0.542} & \revised{0.629} & \revised{0.660} & \revised{0.671} & \revised{0.710}
& \revised{0.685} & \revised{0.783} & \revised{0.754} & \revised{0.771} & \revised{0.807}\\

CART 
& 0.351 & 0.602 & \textbf{0.786} & 0.803 & 0.801
& 0.519 & 0.546 & 0.624 & 0.663 & 0.686
& 0.639 & 0.723 & 0.743 & 0.705 & 0.757 \\

XGBoost
& 0.460 & 0.434 & 0.744 & 0.784 & 0.808
& 0.690 & 0.690 & 0.626 & 0.708 & 0.708
& 0.600 & 0.600 & 0.706 & 0.706 & 0.762 \\

\revised{LightGBM}
& \revised{0.378} & \revised{0.578} & \revised{0.780} & \revised{0.792} & \revised{0.803}
& \revised{0.600} & \revised{0.543} & \revised{0.631} & \revised{0.676} & \revised{0.680}
& \revised{0.652} & \revised{0.733} & \revised{0.759} & \revised{0.717} & \revised{0.769}\\

IO-Tree
& 0.378 & 0.439 & 0.519 & 0.511 & 0.516
& 0.525 & 0.539 & 0.605 & 0.649 & 0.573
& 0.661 & 0.723 & 0.671 & 0.743 & 0.711 \\

CoT-Tree
& 0.407 & 0.461 & 0.617 & 0.559 & 0.591
& 0.665 & 0.701 & 0.700 & 0.696 & 0.733
& 0.699 & 0.739 & 0.705 & 0.740 & 0.731 \\

ToT-Tree
& 0.379 & 0.353 & 0.357 & 0.370 & 0.411
& 0.645 & 0.561 & 0.583 & 0.563 & 0.551
& 0.673 & 0.727 & 0.722 & 0.695 & 0.720 \\

\revised{FeatLLM}
& \revised{0.539} & \revised{0.571} & \revised{0.717} & \revised{0.761} & \revised{0.781}
& \revised{0.658} & \revised{0.700} & \revised{0.717} & \revised{0.725} & \revised{0.730}
& \revised{0.769} & \revised{0.802} & \revised{0.814} & \revised{0.813} & \revised{0.837}\\

\revised{DeLTa}
& \revised{0.357} & \revised{0.448} & \revised{0.577} & \revised{0.673} & \revised{0.729}
& \revised{0.310} & \revised{0.541} & \revised{0.624} & \revised{0.637} & \revised{0.630}
& \revised{0.394} & \revised{0.747} & \revised{0.747} & \revised{0.811} & \revised{0.823}\\

\revised{GPTree}
& \revised{0.326} & \revised{0.479} & \revised{0.780} & \revised{0.741} & \revised{0.720}
& \revised{0.515} & \revised{0.616} & \revised{0.646} & \revised{0.643} & \revised{0.594}
& \revised{0.641} & \revised{0.661} & \revised{0.677} & \revised{0.746} & \revised{0.644}\\

\addlinespace[0.3em]

\textbf{LLMT (Ours)}
& \textbf{0.762} & \textbf{0.762} & \underline{0.784} & \textbf{0.818} & \underline{0.840}
& 0.662 & \textbf{0.764} & \textbf{0.770} & \textbf{0.770} & \textbf{0.770}
& \textbf{0.782} & \textbf{0.811} & \textbf{0.836} & \textbf{0.839} & \textbf{0.843}\\

\bottomrule
\end{tabular}%
}
\caption{Mean accuracy over ten runs across varying \#shots (standard deviation is available at Appendix~\ref{app:additional_results}). DirectZSDT and StepZSDT are constant across \#shots since they cannot utilize any training examples.}
\label{tab:accuracy}
\end{table*}

\subsection{Effectiveness}

We evaluate LLMT against all baselines in terms of classification accuracy. 
Table~\ref{tab:dataset_method_results} reports the results with \#shots=2 per class, while Table~\ref{tab:accuracy} reports results under varying numbers of training examples on three representative datasets.

LLMT consistently outperforms almost all baselines, with more than 10\% improvement over the strongest baseline in many settings. 
It also achieves robust gains over recent LLM-assisted methods such as FeatLLM, DeLTa, and GPTree: FeatLLM and DeLTa tend to overfit in extreme few-shot regimes, while GPTree can be affected by data-induced bias due to its reliance on large-scale training data. 
These results show the benefit of leveraging LLM knowledge in a structured, data-aware manner. 
In addition, ToT-Tree often performs worse than CoT-Tree, suggesting that recursive node-level prompting introduces noise and inconsistency for structured tabular data. 
By contrast, LLMT uses LLM-generated rules as decision units, leading to more effective tree construction.

\camerarevise{We further assess high-dimensional performance on three datasets with over 100 attributes. These results show that LLMT maintains its advantage in high-dimensional feature spaces and further support the effectiveness of its ensemble extension. Please refer to Appendix~\ref{app:high_dimensional} for detailed results.}

\subsection{Efficiency}

\begin{table*}[!t]
\centering
\resizebox{\textwidth}{!}{%
\begin{tabular}{@{}lllllll@{}}
\toprule
\multirow{2}{*}{} 
& \multicolumn{3}{l}{Time (s)}                
& \multicolumn{3}{l}{Token (\#Input token, \#Output token, \#Total)} \\ 
\cmidrule(l){2-4} \cmidrule(l){5-7} 
& Diabetes & Nursery & Spambase  
& Diabetes & Nursery & Spambase \\ 
\midrule

\revise{DirectZSDT}   
& \revise{43.64}
& \revise{31.75}
& \revise{55.33}
& \revise{(2.9k, 1.1k, 4k)}
& \revise{(3.8k, 1.3k, 4.1k)}
& \revise{(3.4k, 1.4k, 4.8k)} \\ 
\midrule

\revise{StepZSDT} 
& \revise{1006.67}
& \revise{1191.40}
& \revise{4036.18}
& \revise{(85.6k, 22.1k, 107.7k)}
& \revise{(128.3k, 27.2k, 155.5k)}
& \revise{(358.9k, 95.4k, 454.3k)} \\ 
\midrule

CoT-Tree         
& 17.00
& 9.58
& 12.05
& (2.7k, 1.3k, 4k)
& (2.7k, 1.0k, 3.7k)
& (5.6k, 0.9k, 6.5k) \\ 
\midrule

ToT-Tree         
& 33.39
& 33.22
& 64.56
& (32.9k, 0.5k, 33.4k)
& (35.4k, 0.6k, 36k)
& (140k, 1.3k, 141.3k) \\ 
\midrule

\revised{FeatLLM}
& \revised{259.07}
& \revised{416.16}
& \revised{324.06}
& \revised{(22.2k, 10.3k, 32.5k)}
& \revised{(24.4k, 16.3k, 40.7k)}
& \revised{(48.5k, 12.3k, 60.8k)} \\
\midrule

\revised{DeLTa} 
& \revised{30.20} 
& \revised{26.90} 
& \revised{28.84} 
& \revised{(3.1k, 3.3k, 6.4k)} 
& \revised{(3.1k, 2.8k, 5.9k)} 
& \revised{(5.2k, 3.1k, 8.3k)} \\
\midrule

\revised{GPTree}
& \revised{140.75}
& \revised{103.47}
& \revised{125.56}
& \revised{(12.5k, 2.3k, 14.8k)}
& \revised{(9.5k, 3.2k, 12.7k)}
& \revised{(27.6k, 2.5k, 30.0k)} \\
\midrule

LLMT              
& 20.48
& 16.33
& 18.18
& (3.1k, 0.2k, 3.3k)
& (3.4k, 0.2k, 3.6k)
& (6.0k, 0.4k, 6.4k) \\ 
\midrule

\tb{Savings}                  
& \cellcolor{gray!20}\tb{1.63x}
& \cellcolor{gray!20}\tb{2.03x}
& \cellcolor{gray!20}\tb{3.56x}
& \cellcolor{gray!20}\tb{10.1x}
& \cellcolor{gray!20}\tb{10x}
& \cellcolor{gray!20}\tb{22.1x} \\ 

\bottomrule
\end{tabular}%
}
\caption{The training time (s) and \#token in the \#shots=32 setting. The savings of LLMT are computed against ToT-Tree in terms of \#total tokens.}
\label{tab:efficiency}
\end{table*}
We evaluate computational efficiency by measuring \revised{training time} and token usage per prediction on three tabular datasets, as summarized in Table~\ref{tab:efficiency}. 
We omit CART, XGBoost, and IO-Tree because they are already highly efficient but have relatively poor classification performance. 
DirectZSDT incurs high latency due to its two-stage dialogue, where the LLM first generates a tree and then converts it into executable Python code. 
StepZSDT is even more expensive, as it prompts the LLM to generate feature-specific rules for each node.

Compared with CoT-Tree, LLMT achieves comparable \revised{training time} and token usage. \camerarevise{Compared with ToT-Tree, LLMT avoids repeated planning and voting, reducing token usage by 10--22$\times$ and construction time by 1.6--3.6$\times$.}
LLMT also has comparable computational cost to DeLTa, while being up to 10$\times$ faster than FeatLLM and more token-efficient than GPTree, which rely on repeated rule-parsing queries and sequential batch-wise insight generation, respectively.

\begin{figure*}[!t]
    \centering
    \includegraphics[width=0.85\textwidth]{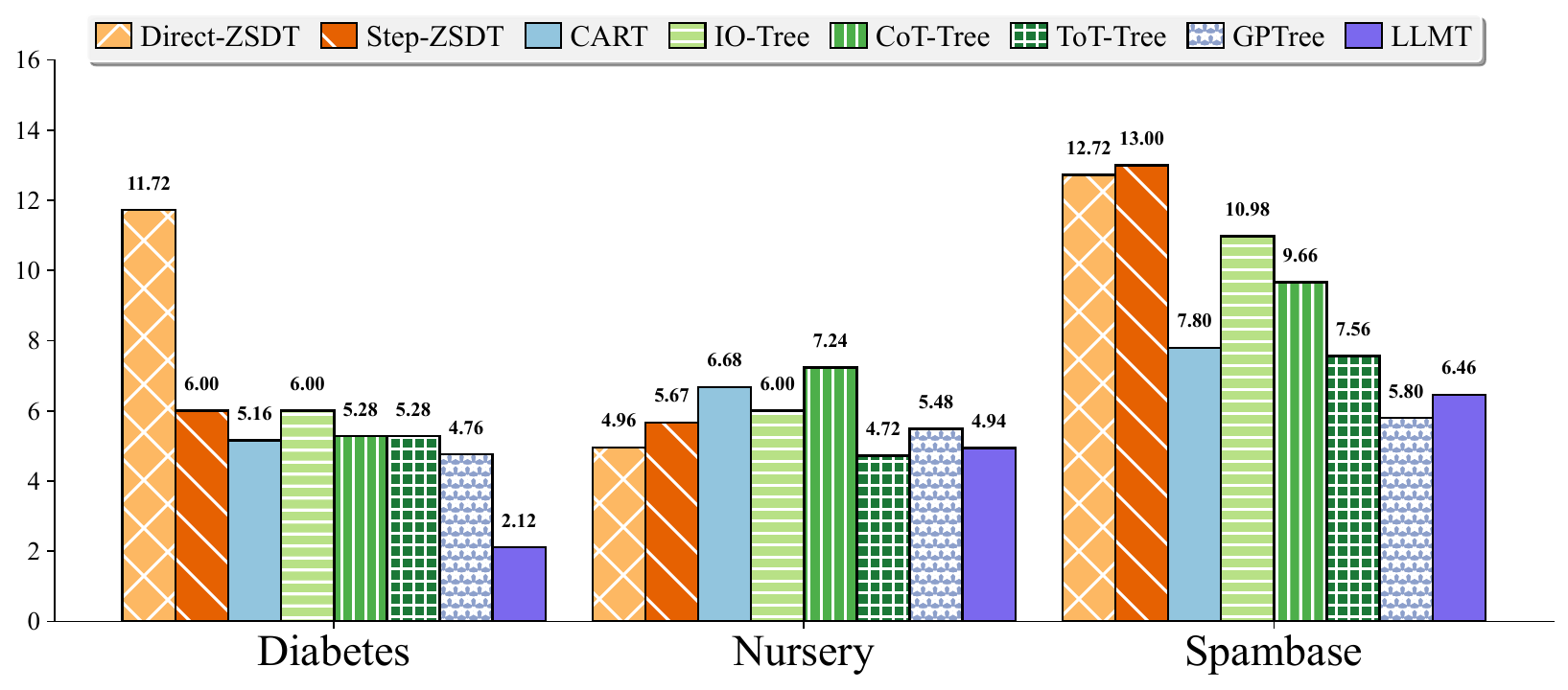}
    % \vspace{-10pt}
    \caption{\camerarevise{Average number of nodes in trees generated by each method. Fewer nodes indicate a more concise tree.}}
    \label{fig:interpretability}
\end{figure*}

\begin{table*}[!t]
    \centering
    \renewcommand{\arraystretch}{1.1}
    \resizebox{\textwidth}{!}{%
    \begin{tabular}{@{}lcccccccc@{}}
    \toprule
    \tb{Dataset} & \tb{Direct-ZSDT} & \tb{Step-ZSDT} & \tb{CART} & \tb{IO-Tree} & \tb{CoT-Tree} & \tb{ToT-Tree} & \tb{GPTree} & \tb{LLMT} \\
    \midrule
    Diabetes & 0.422 & 0.420 & 0.414 & 0.397 & 0.397 & 0.403 & 0.425 & \textbf{0.396} \\
    Nursery  & 0.665 & 0.634 & 0.516 & 0.582 & 0.534 & 0.629 & 0.552 & \textbf{0.456} \\
    Spambase & 0.378 & 0.478 & 0.360 & \textbf{0.341} & 0.361 & 0.351 & 0.396 & 0.352 \\
    \midrule
    \tb{Average} & 0.488 & 0.511 & 0.430 & 0.440 & 0.431 & 0.461 & 0.458 & \textbf{0.401} \\
    \bottomrule
    \end{tabular}%
    }
    \caption{Average Gini impurity of rules across different methods. Lower values indicate better rule quality.}
    \label{tab:gini_impurity_main}
\end{table*}

\subsection{Interpretability}

\camerarevise{Decision trees are intrinsically interpretable because each prediction follows an explicit root-to-leaf path, making the decision process directly inspectable. Compared with coefficient-based, hybrid, or black-box models, standalone trees generally provide more intuitive explanations for individual predictions. We therefore compare methods that output standalone trees in terms of structural conciseness and split purity, measured by average node count and average Gini impurity, as reported in Figure~\ref{fig:interpretability} and Table~\ref{tab:gini_impurity_main}, respectively.}

\camerarevise{LLMT produces the most concise trees on Diabetes and remains among the most compact methods on Nursery and Spambase. It also achieves the lowest average Gini impurity, performing best on Diabetes and Nursery and close to the best method on Spambase. Overall, LLMT retains the intrinsic readability of a standalone decision tree while generating compact structures with high-quality rules.}

\subsection{Ablation and Sensitivity Study}
\label{sec:ablation}
\begin{figure*}[!]
    \centering
    \begin{subfigure}[b]{0.47\textwidth}
        \centering
        \includegraphics[width=\textwidth]{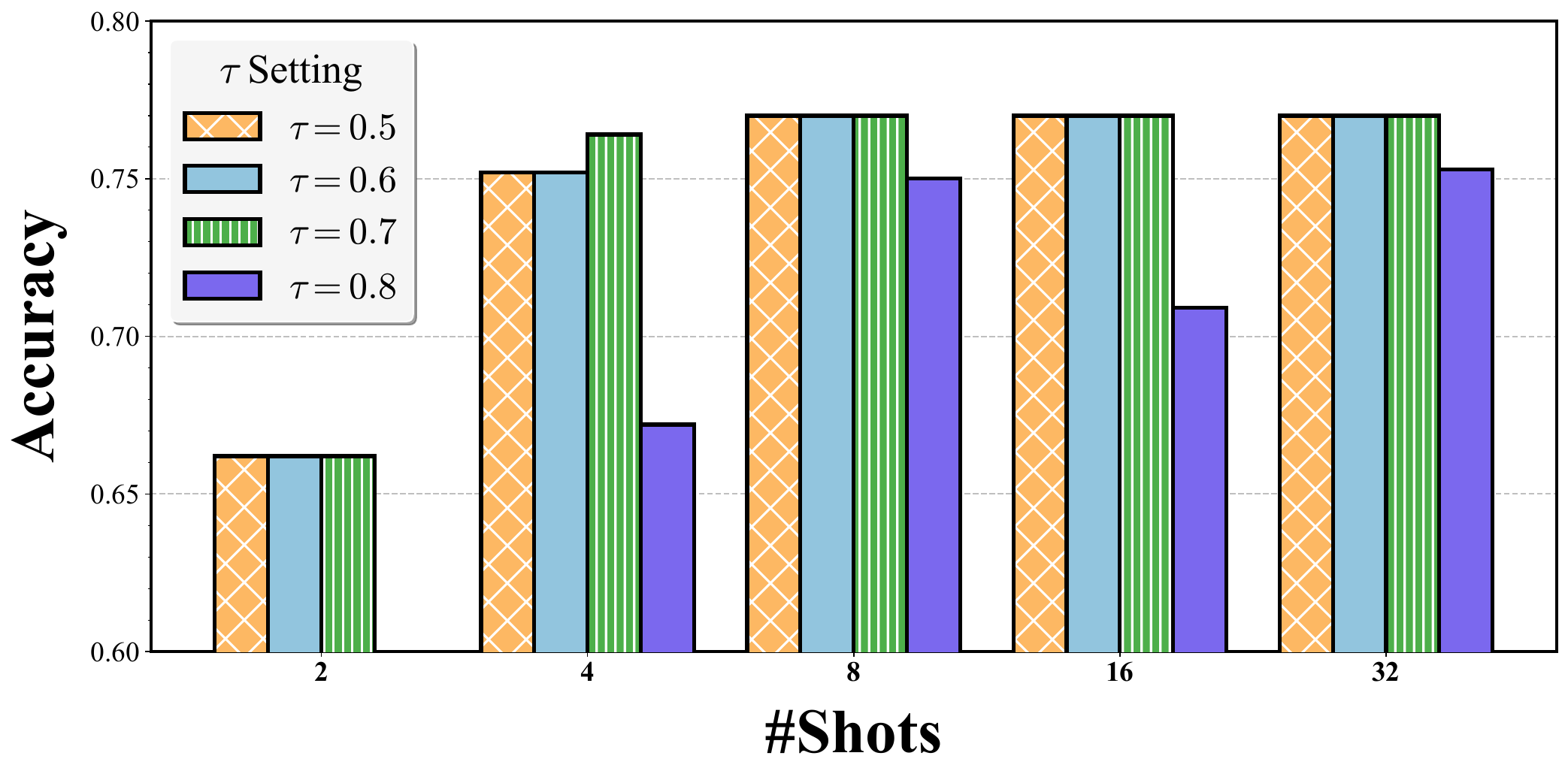} % 固定相同高度
        \caption{Diabetes}
        \label{fig:dia_sen}
    \end{subfigure}
    % \hfill
    \begin{subfigure}[b]{0.47\textwidth}
        \centering
        \includegraphics[width=\textwidth]{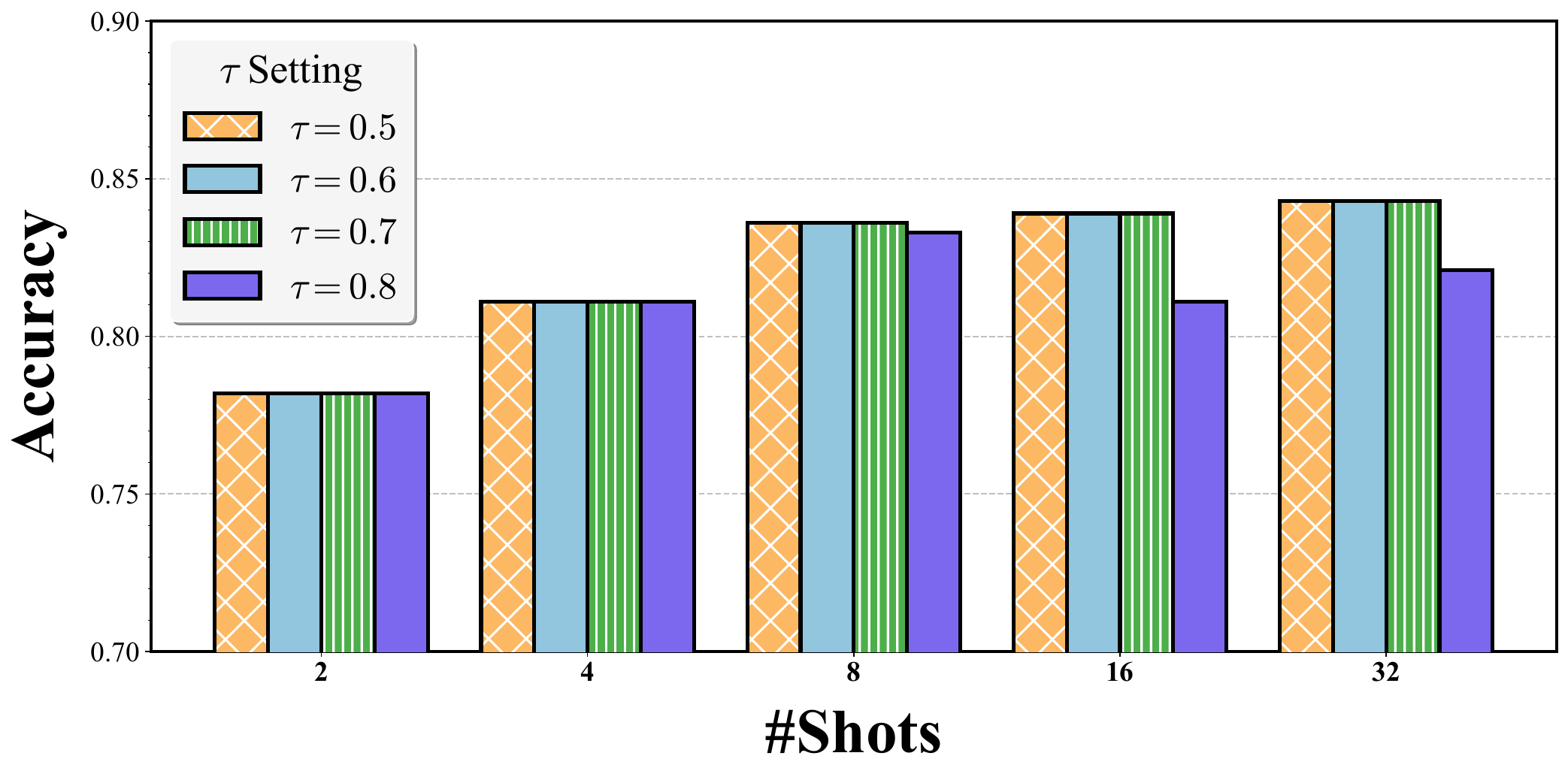} % 固定相同高度
        \caption{Spambase}
        \label{fig:spam_sen}
    \end{subfigure}
    % \vspace{-10pt}
    \caption{The accuracy of LLMT with different $\tau$}
    \label{fig:sensitivity}
\end{figure*}

\begin{figure*}[!]
    \centering
    % \begin{subfigure}[b]{0.32\textwidth} 
    %     \centering
    %     \includegraphics[width=\textwidth]{figures/nursery_refinement.pdf} % 固定高
    %     \caption{Nursery.}
    %     \label{fig:reasoning}
    % \end{subfigure}
    % \hfill
    \begin{subfigure}[b]{0.49\textwidth}
        \centering
        \includegraphics[width=\textwidth]{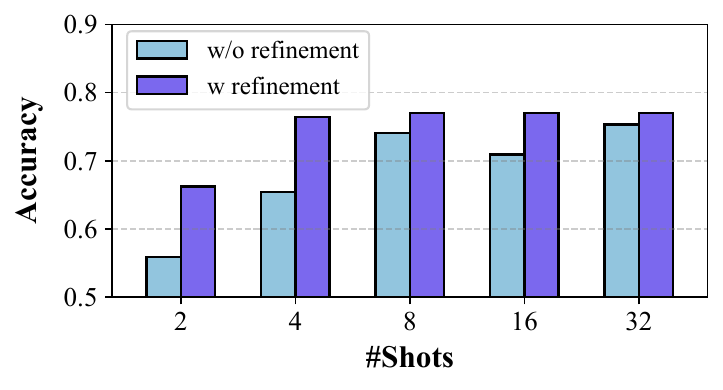} % 固定相同高度
        \caption{Diabetes}
        \label{fig:dia_ref}
    \end{subfigure}
    % \hfill
    \begin{subfigure}[b]{0.49\textwidth}
        \centering
        \includegraphics[width=\textwidth]{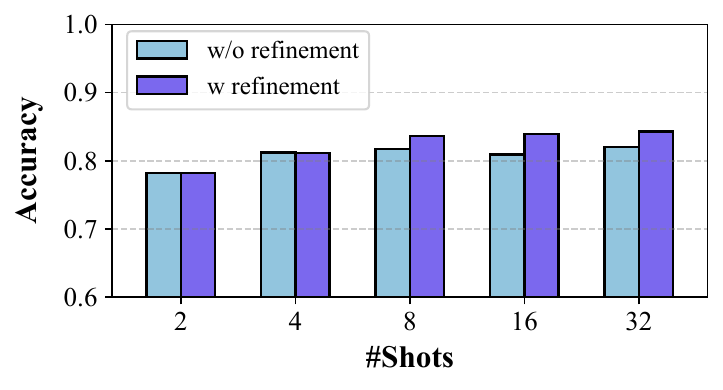} % 固定相同高度
        \caption{Spambase}
        \label{fig:spam_ref}
    \end{subfigure}
    % \vspace{-10pt}
    \caption{The accuracy of LLMT with/without leaf refinement}
    \label{fig:ablation}
\end{figure*}
We further analyze the effect of the \textit{Leaf Refinement} stage and its confidence threshold $\tau$. 
The threshold $\tau$ controls whether the LLM accepts a leaf node's decision based on its confidence in the corresponding reasoning path. 
As shown in Figure~\ref{fig:sensitivity}, LLMT is generally stable across different $\tau$ values, and $\tau=0.7$ usually achieves the best performance, suggesting a good balance between preserving the initial tree structure and applying LLM-based corrections. 
It indicates that LLMT does not require extensive hyperparameter tuning.

We then remove \textit{Leaf Refinement} to evaluate its contribution on Diabetes and Spambase. Figure~\ref{fig:ablation} shows consistent accuracy gains of 8\% and 2\%, respectively. Beyond accuracy, LLM verification may identify same-label sibling leaves for merging, simplifying the tree and mitigating spurious splits caused by limited few-shot data.

\section{Conclusion}

In this paper, we present a novel framework for distilling the knowledge of LLMs into interpretable and efficient decision trees for tabular data classification. Unlike prior approaches that prompt LLMs to directly generate full decision structures which often leads to instability and high inference costs, our method utilizes the ability of LLMs to generate discrete rules and organizes them as a tree statistically.
By bridging the strengths of LLMs and symbolic models, our method offers a promising direction for building trustworthy AI systems that can reason over structured data.

\section*{Limitations}

Despite its strengths, our method assumes that the LLM possesses sufficient background knowledge relevant to the target dataset. If the dataset lies outside the knowledge scope of the LLM, then the generated rules—and consequently the constructed tree—may be inaccurate or ineffective. Moreover, our method inherits any biases in the underlying LLM, which may influence the rules extracted and affect downstream decisions. Caution must be exercised when deploying the distilled trees in high-stakes applications, and future work should explore fairness-aware prompting and post-hoc auditing. 

\section*{Acknowledgements}
This research is supported in part by National Natural Science Foundation of China (Grant No. 62502174). 
This research is also supported in part by the National Research Foundation, Singapore and Infocomm Media Development Authority under its Trust Tech Funding Initiative. Any opinions, findings and conclusions or recommendations expressed in this material are those of the author(s) and do not reflect the views of National Research Foundation, Singapore and Infocomm Media Development Authority.

\bibliography{main}

\clearpage
\newpage
\appendix
% Shared widths keep all appendix tables aligned with the two reference tables:
% Table~\ref{tab:newer_masking} for full-width floats and
% Table~\ref{tab:datasets} for single-column floats.
\newlength{\appendixfulltablewidth}
\setlength{\appendixfulltablewidth}{\textwidth}
\newlength{\appendixsinglecoltablewidth}
\setlength{\appendixsinglecoltablewidth}{\columnwidth}

\section{Proof}
\renewcommand\thesection{4}
\label{app:proof}

\begin{proposition}[Expected Impurity Reduction]
Let \(r^*\in\mathcal{R}\) be the selected rule at node \(t\), splitting dataset \(D\) into \(D_l,D_r\). Then, the expected impurity after splitting satisfies:
\[
\E[\mathrm{Gini}(D_l,D_r)] \leq \mathrm{Gini}(D),
\]
with equality if and only if the split yields no information gain.
\end{proposition}

\begin{proof}
Define the Gini impurity for dataset \(D\) with \(C\) classes as:
\[
\mathrm{Gini}(D)=1-\sum_{i=1}^{C}p_i^2,\quad p_i=\frac{|D^{(i)}|}{|D|}.
\]

The selected rule \(r^*\) minimizes:

\[
\E[\mathrm{Gini}(D_l,D_r)]
= \frac{|D_l|}{|D|}\mathrm{Gini}(D_l)
\!+\!
\frac{|D_r|}{|D|}\mathrm{Gini}(D_r).
\]

Since \(r^*\) minimizes the above expression, it must hold that:
\[
\E[\mathrm{Gini}(D_l,D_r)]\leq\mathrm{Gini}(D).
\]

Equality occurs precisely when no effective class separation occurs, resulting in no impurity reduction.
\end{proof}

\begin{theorem}[Generalization Error Bound]
\label{thm:gen-appendix}
Consider a decision tree \(T\) constructed using LLMT of maximum depth \(d\) from a finite rule set \(\mathcal{R}\) with size \(K\). Given \(n\) training samples, for any \(\delta>0\), with probability at least \(1-\delta\), the generalization error \(R(T)\) satisfies:
\par\nobreak
\begingroup\small
\setlength{\abovedisplayskip}{2pt plus 1pt minus 2pt}
\setlength{\abovedisplayshortskip}{0pt plus 1pt}
\setlength{\belowdisplayskip}{8pt plus 2pt minus 2pt}
\setlength{\belowdisplayshortskip}{5pt plus 2pt minus 2pt}
\begin{align}
R(T) &\;\le\; \hat R(T)
\;+\; \sqrt{\frac{2\Bigl(\log P(K,2^d-1)+\log \tfrac{2}{\delta}\Bigr)}{n}}\\
&\;\le\; \hat R(T) \;+\;
\sqrt{\frac{2\Bigl((2^d-1)\log K+\log \tfrac{2}{\delta}\Bigr)}{n}},
\label{eq:genbound-appendix}
\end{align}
\endgroup\nobreak
where $P(K,2^d-1)=K!/(K-(2^d-1))!$.
\end{theorem}

\begin{proof}
Let $N_{\mathrm{int}}=2^d-1$, which is the number of internal nodes of a tree with depth $d$. We first count the number of possible trees produced by LLMT.
At each internal node, the algorithm selects a distinct rule from 
the global pool $\mathcal{R}$ of size $K$, without replacement.
Thus, the number of possible ordered selections of $N_{\mathrm{int}}$ distinct rules is
\[
P(K,N_{\mathrm{int}}) \;=\; \frac{K!}{(K-N_{\mathrm{int}})!}.
\]
Each selection uniquely determines the split structure of the tree,
since the structure is fixed (binary tree of maximum depth $d$) 
and rules are not reused globally.

Next, we consider leaf labels.
In LLMT, leaf predictions are determined by the \emph{leaf-refinement step}:
given the path of rules leading to the leaf, the LLM outputs a 
confidence distribution, and the label is set using a fixed threshold~$\tau$.
This procedure depends only on the rule path and $\tau$, not on the training data.
Therefore, once the split structure is fixed, the leaf labels are fixed as well,
and there is no additional multiplicative factor for label assignments.

Let $\mathcal{H}$ be the set of all trees that can be produced by the algorithm.
Then $|\mathcal{H}|\le P(K,N_{\mathrm{int}})$.
For each tree $T\in \mathcal{H}$, let $\hat R(T)$ be its empirical error on $n$ i.i.d.\ samples 
and $R(T)$ its population error.
By Hoeffding's inequality,
\[
\Pr\bigl(|R(T)-\hat R(T)| \ge \varepsilon\bigr) \;\le\; 2\,e^{-2n\varepsilon^2}.
\]
Applying the union bound over all $T\in\mathcal{H}$ gives
\[
\Pr\Bigl(\sup_{T\in\mathcal{H}} |R(T)-\hat R(T)| \ge \varepsilon\Bigr)
\;\le\; 2\,|\mathcal{H}|\,e^{-2n\varepsilon^2}.
\]
Setting the right-hand side to $\delta$ and solving for $\varepsilon$ yields
\begin{align}
\varepsilon 
&\;=\; \sqrt{\frac{2\bigl(\log |\mathcal{H}| + \log(2/\delta)\bigr)}{n}}\\
&\;\le\; \sqrt{\frac{2\bigl(\log P(K,N_{\mathrm{int}}) + \log(2/\delta)\bigr)}{n}}.
\end{align}
This proves the first inequality in~\eqref{eq:genbound-appendix}.
The second inequality follows from the fact that
\[
\log P(K,N_{\mathrm{int}})
= \sum_{t<N_{\mathrm{int}}}\log(K-t)
\le N_{\mathrm{int}}\log K.
\]

\end{proof}

\renewcommand\thesection{\Alph{section}}
\setcounter{section}{1}

\section{Algorithm}

\camerarevise{We summarize the three-stage LLMT framework in Algorithm~\ref{alg:distillation} and its optional ensemble extension, LLMT Forest, in Algorithm~\ref{alg:llmt_forest}. Algorithm~\ref{alg:distillation} presents rule generation (lines 1--2), tree assembly (lines 3--12), and leaf refinement (lines 13--20), while Algorithm~\ref{alg:llmt_forest} describes random feature-subspace sampling, bootstrap sampling, independent LLMT tree construction, and majority-vote aggregation.}

\begin{algorithm}[t]
\DontPrintSemicolon
\SetNoFillComment
\LinesNumbered
\SetArgSty{textnormal}
\KwIn{Training samples $\mc{D}$, depth $d$, LLM $f_\theta(\cdot)$}
\KwOut{Decision Tree $T$}
\BlankLine

\textbf{Step 1 - Rule Generation:} \\
$\mc{R} \leftarrow f_\theta(P_{rule}) \quad \text{where } \mc{R} = [(r_k, s_k)]_{k=1}^K$

\BlankLine

% \begin{minipage}[t]{0.48\textwidth}
\textbf{Step 2 – Tree Assembly:}\\
\For{$i = 1,\dots,d-1$}{
    \For{$j = 2^{\,i-1},\dots,2^i-1$}{
        $R_{\mathrm{path}} \leftarrow \{\text{rules used until node }j\}$;\\
        $\mc{R}_{\mathrm{candidate}} \leftarrow \mc{R} \setminus R_{\mathrm{path}} $ (group by $s_k$);\\
        \For{$G \in \mc{R}_{\mathrm{candidate}}$}{
            select $r^* \in G$ by Eq.~\eqref{eq:rule_selection} on $\mathcal{D}_j$\\
            \If{Gini‐gain$(r^*,\mathcal{D}_j)>0$}{
                 $T[j]\leftarrow r^*$;\\
                 \textbf{break};
            }
        }
    }
}
% \vspace{10pt}
% \end{minipage}
% \hfill
% \BlankLine
% \begin{minipage}[r]{0.48\textwidth}

\textbf{Step 3 – Leaf Refinement:}\\
\ForEach{leaf node $T[i]\in T$}{
    % Evaluate $p$ using $\mathcal{M}$ to obtain confidence $c_p$;\\
    $C_p \leftarrow f_\theta(T[i]; P_{refine})$\;
    $y^* \leftarrow \arg\max \mathbf{C_p}$\quad\;
    $c^* \leftarrow \max \mathbf{C_p}$\;
    \If{$c^*>\tau \;\wedge\; y^*\neq T[i]$}{
        $T[i]\leftarrow y^*$
    }
}
\Return $T$;
% \vspace{20pt}

% \vspace{10pt}
% \end{minipage}
\caption{The process of LLMT}
\label{alg:distillation}
\end{algorithm}

\begin{algorithm}[t]
\begin{camerarevision}
\DontPrintSemicolon
\SetNoFillComment
\LinesNumbered
\SetArgSty{textnormal}
\KwIn{Training samples $\mc{D}$, feature set $\mc{F}$, number of trees $M$, depth $d$, LLM $f_\theta(\cdot)$}
\KwOut{LLMT Forest $\mc{T}$}
\BlankLine

$q \leftarrow \left\lceil \sqrt{|\mc{F}|} \right\rceil$\;

\For{$m=1,\ldots,M$}{
    $\mc{F}_m \leftarrow \operatorname{RandomSubset}(\mc{F},q)$\;
    $\mc{D}_m \leftarrow
    \operatorname{Project}_{\mc{F}_m}
    \bigl(\operatorname{Bootstrap}(\mc{D})\bigr)$\;
    $T_m \leftarrow
    \operatorname{LLMT}(\mc{D}_m,d,f_\theta(\cdot))$\;
}

$\mc{T}(x) \leftarrow
\operatorname{MajorityVote}
\bigl(\{T_m(x|_{\mc{F}_m})\}_{m=1}^{M}\bigr)$\;

\Return $\mc{T}$\;
\caption{The process of LLMT Forest}
\label{alg:llmt_forest}
\end{camerarevision}
\end{algorithm}

\section{\revise{Theoretical Time Complexity Analysis}}
\label{app:time_complexity}

In practice, dominant runtime cost of our method comes from LLM calls; the tree-assembly steps is lightweight and takes less than 5\% of the total wall-clock time across all datasets. Here we provide a theoretical analysis of the time complexity.

Let $n$ be the number of training examples, $K$ the number of generated rules, $d_{\max}$ the maximum tree depth, and $T_{\text{LLM}}$ the time of a single LLM call. 

\textbf{Rule Generation and Leaf Refinement:} For rule generation, we need to generate rules for each internal node. In the worst case, a complete binary tree of depth $d_{\max}$ has $2^{d_{\max}} - 1$ internal nodes, requiring at most $2^{d_{\max}} - 1$ LLM calls. For leaf refinement, we need to evaluate each leaf node, which requires at most $2^{d_{\max}}$ LLM calls (one per leaf). Therefore, the total cost for LLM calls is at most $(2^{d_{\max}} + 1) T_{\text{LLM}}$.

\textbf{Tree Assembly:} For tree assembly, at each internal node, we evaluate at most $K$ rules on its local subset of examples. Since each example appears in at most $d_{\max}$ nodes along its path from root to leaf, the total number of rule evaluations is at most $K \cdot n \cdot d_{\max}$. Each rule evaluation involves checking the condition against an example, which takes constant time. Therefore, the total cost for tree assembly is $O(K n d_{\max})$.

\textbf{Total Time Complexity:} Combining both components, the total time complexity is:
\[
(2^{d_{\max}} + 1) T_{\text{LLM}} + O(K n d_{\max}).
\]

Since $T_{\text{LLM}}$ is typically much larger than the time for a single rule evaluation, and $K$ is set to $\max(10, 2^{d_{\max}} - 1)$ (as described in Section~\ref{sec:exp}), the LLM call cost dominates the overall runtime in practice.

\section{\revise{Step-by-Step Example of LLMT}}
\label{app:example}

We provide a concrete example below to illustrate how LLMT works step-by-step. This example uses the Diabetes dataset with 4 training samples, maximum depth 3, and refinement threshold $\tau=0.7$.

\textbf{Training samples (toy subset):}
\begin{itemize}
    \item Sample 1: [Glucose=78, BMI=31.2, Age=42, DPF=0.382, \ldots] $\rightarrow$ no
    \item Sample 2: [Glucose=64, BMI=29.2, Age=21, DPF=0.192, \ldots] $\rightarrow$ no
    \item Sample 3: [\ldots] $\rightarrow$ yes
    \item Sample 4: [\ldots] $\rightarrow$ yes
\end{itemize}

\textbf{Step 1 -- Rule generation.}

The LLM, given only task metadata, outputs meta-rules with confidence scores, e.g.:
\begin{itemize}
    \item \texttt{Glucose < 140 [confidence: 10]}
    \item \texttt{BMI < 30.0 [confidence: 9]}
    \item \texttt{Age < 30 [confidence: 8]}
    \item \texttt{DPF < 0.5 [confidence: 8]}
    \item \ldots
\end{itemize}

\textbf{Step 2 \& 3 -- Tree assembly + leaf refinement.}

At the root, we evaluate candidate rules by Gini gain and select \texttt{BMI < 30.0}. This splits the data into a pure left node (1 sample, class: no) and a mixed right node (3 samples). The left node becomes a leaf; we send its path (\texttt{BMI < 30.0}) to the LLM, which returns calibrated probabilities (e.g., no: 0.700, yes: 0.300), so the leaf remains \texttt{no} since confidence $\geq \tau$. For the right node, we again select the rule with the highest Gini gain (e.g., \texttt{DPF < 0.5}) and recurse until either purity or depth constraints are met. For each resulting leaf path (e.g., \texttt{BMI $\geq$ 30.0 AND DPF $\geq$ 0.5}), we perform the same refinement step via the LLM and only change the leaf label when confidence exceeds $\tau$.

The final toy tree has the form of Figure \ref{fig:llmt-example-step}:
\begin{figure}[h]
    \centering
    \includegraphics[width=0.65\columnwidth]{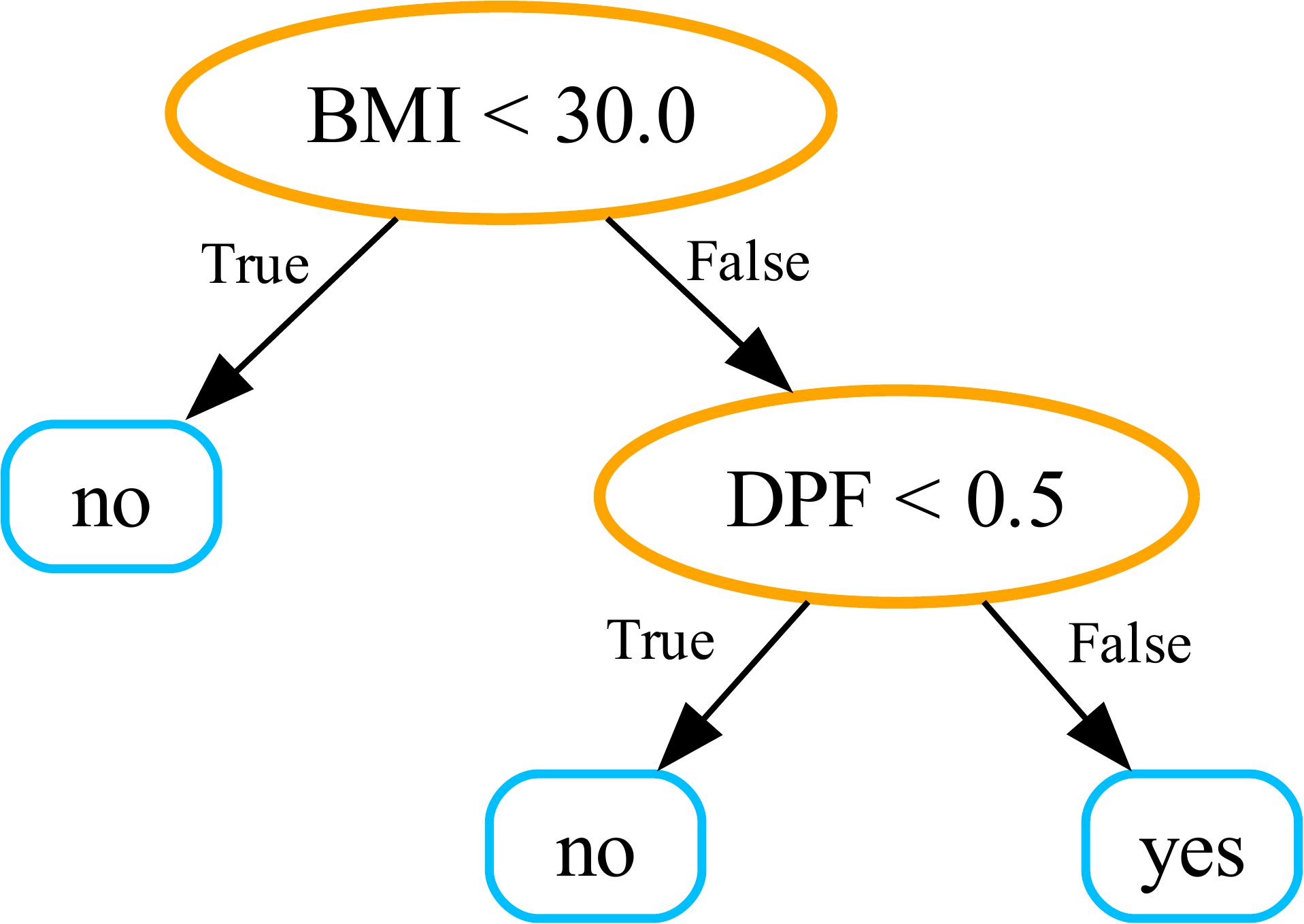}
    \caption{LLMT Example Tree (Step-by-Step)}
    \label{fig:llmt-example-step}
\end{figure}

\section{\revise{Information of Datasets}}
\label{app:data}

The details of the evaluated datasets are presented below:  
1) \textbf{Nursery}~\cite{nursery_76}: A multi-class classification dataset originally used to evaluate nursery school application outcomes;  
2) \textbf{Diabetes}~\cite{diabetes_34}: A binary classification dataset from the Pima Indians Diabetes Database to predict the onset of diabetes;  
3) \textbf{Spambase}~\cite{spambase_94}: A binary classification dataset from the UCI repository to determine whether an email is spam;  
4) \textbf{Abalone}~\cite{abalone_94}: A binary dataset aiming to predict the age of abalones (measured by the number of rings);  
5) \textbf{Blood}~\cite{blood_08}: A binary classification dataset on blood donation behavior to predict whether a donor will give blood in the future based on historical donation patterns;  
6) \textbf{Iris}~\cite{iris_36}: A classical multi-class dataset consisting of flower measurements for classifying iris species;
\revise{
7) \textbf{Breast}~\cite{breast_resistin_2018}: A binary classification dataset using resistin, glucose, age, and BMI to predict breast cancer presence;
8) \textbf{Glioma}~\cite{glioma_voting_2022}: A binary classification dataset for glioma grading based on clinical and molecular characteristics; }
\camerarevise{9) \textbf{Communities}~\citep{redmond2002communities}: A high-dimensional classification dataset with 103 features for predicting violent-crime risk tiers;  
10) \textbf{Ecom}~\citep{tabred2024}: A 119-feature coupon-redemption dataset derived from the Acquire Valued Shoppers data using the TabReD preprocessing pipeline;  
11) \textbf{Myocardial}~\citep{golovenkin2020myocardial}: A 111-feature clinical dataset for identifying chronic heart failure after myocardial infarction.}
Among them, we select Diabetes, Spambase, and Nursery as the major datasets for evaluation across various settings. These datasets vary widely in sample size, feature dimensionality, and number of classes, providing a comprehensive testbed.

We list all the dataset information in Table \ref{tab:datasets}. All datasets are publicly available. The UCI-hosted datasets used in this paper are released under the Creative Commons Attribution 4.0 license (CC BY 4.0), which permits reuse and redistribution with proper attribution.

\begin{table}[h]
    \centering
    \setlength{\tabcolsep}{5pt}
    \begin{tabular*}{\appendixsinglecoltablewidth}{@{\extracolsep{\fill}}lccc@{}}
    \toprule
    Dataset                  & \#Instances & \#Features & \#Classes \\ \midrule
    Nursery                  & 12,960      & 8          & 3         \\
    Diabetes                 & 768         & 8          & 2         \\
    Spambase                 & 4,601       & 18         & 2         \\
    Abalone                  & 4,177       & 8          & 2         \\
    Blood                    & 748         & 4          & 2         \\
    Iris                     & 150         & 4          & 3         \\
    % Bankruptcy               & 78,682      & 19         & 2         \\
    Breast                   & 116         & 9          & 2         \\
    % Penguins                 & 333         & 7          & 3         \\
    Glioma                   & 839         & 23         & 2         \\
    \camerarevise{Communities} & \camerarevise{1,994} & \camerarevise{103} & \camerarevise{3} \\
    \camerarevise{Ecom}       & \camerarevise{160,057} & \camerarevise{119} & \camerarevise{2} \\
    \camerarevise{Myocardial} & \camerarevise{686} & \camerarevise{111} & \camerarevise{2} \\ \bottomrule
    \end{tabular*}
    \caption{Statistics of datasets used in our experiments.}
    \label{tab:datasets}
\end{table}

\section{\revise{Experimental Settings}}
\label{app:exp_settings}

This section provides adequate experimental parameters for reproducibility. All hyperparameters, model configurations, and experimental settings are detailed below.

\subsection{\revise{Experimental Details for Observation 1 E}}
\label{app:obs1_scenarios}

In Observation 1, we compare LLM performance under two scenarios: \textit{original} and \textit{masked}. In the \textit{original} scenario, the model has access to the full dataset description, including feature names and label names. For example, a feature appears as \texttt{Age} or \texttt{Glucose} with its semantic meaning. In the \textit{masked} scenario, the model does not have access to feature names or label names. Instead, features are represented only as column numbers (e.g., \texttt{Feature 2} or \texttt{Column 3}), and labels are similarly anonymized (e.g., \texttt{Class 0} and \texttt{Class 1} instead of \texttt{yes} and \texttt{no}). This masking removes the semantic context that LLMs rely on for leveraging their background knowledge about the task domain.

\begin{camerarevision}
To reduce dependence on old, widely used benchmarks, we extend the same original-versus-masked protocol to six newer datasets. Breast (2018) and Glioma (2022) already appear in our main evaluation. Paddy was released by UCI in July 2025~\citep{paddy2025}, and Students was released in July 2025~\citep{students2025}; both postdate Llama-4 Maverick, the newest model in this experiment. Loan and Personality are synthetic datasets from the 2025 Kaggle Playground Series~\citep{kaggleloan2025,kagglepersonality2025}, so their generated records could not have appeared in model pretraining. We evaluate the same four LLMs used in Figure~\ref{fig:mot1}. Table~\ref{tab:newer_masking} reports accuracy and the absolute percentage-point reduction caused by masking.
\end{camerarevision}

\begin{table*}[!]
\centering
\begin{camerarevision}
\setlength{\tabcolsep}{2.5pt}
\begin{tabular*}{\appendixfulltablewidth}{@{\extracolsep{\fill}}llrrr@{\hspace{10pt}}llrrr@{}}
\toprule
Dataset & Model & Original & Masked & $\Delta$ & Dataset & Model & Original & Masked & $\Delta$ \\
\midrule
Breast & DeepSeek-V3 & 56.0 & 50.0 & 6.0 & Paddy & DeepSeek-V3 & 34.5 & 31.5 & 3.0 \\
Breast & Llama-4 & 56.9 & 49.1 & 7.8 & Paddy & Llama-4 & 40.0 & 37.5 & 2.5 \\
Breast & Gemma-2 & 55.2 & 52.6 & 2.6 & Paddy & Gemma-2 & 37.0 & 32.5 & 4.5 \\
Breast & Qwen2.5 & 54.3 & 54.3 & 0.0 & Paddy & Qwen2.5 & 32.5 & 30.5 & 2.0 \\
\addlinespace
Glioma & DeepSeek-V3 & 89.0 & 34.0 & 55.0 & Personality & DeepSeek-V3 & 93.0 & 59.0 & 34.0 \\
Glioma & Llama-4 & 89.0 & 35.0 & 54.0 & Personality & Llama-4 & 93.0 & 64.0 & 29.0 \\
Glioma & Gemma-2 & 53.0 & 57.0 & $-4.0$ & Personality & Gemma-2 & 93.0 & 61.5 & 31.5 \\
Glioma & Qwen2.5 & 85.5 & 33.0 & 52.5 & Personality & Qwen2.5 & 93.5 & 81.5 & 12.0 \\
\addlinespace
Loan & DeepSeek-V3 & 68.0 & 52.5 & 15.5 & Students & DeepSeek-V3 & 77.0 & 56.0 & 21.0 \\
Loan & Llama-4 & 68.0 & 49.5 & 18.5 & Students & Llama-4 & 72.5 & 39.0 & 33.5 \\
Loan & Gemma-2 & 54.5 & 44.0 & 10.5 & Students & Gemma-2 & 66.5 & 43.5 & 23.0 \\
Loan & Qwen2.5 & 70.5 & 34.0 & 36.5 & Students & Qwen2.5 & 78.5 & 46.0 & 32.5 \\
\bottomrule
\end{tabular*}
\end{camerarevision}
\caption{\camerarevise{Original-versus-masked zero-shot accuracy (\%) on six newer datasets. $\Delta$ is Original minus Masked in percentage points. Llama-4 denotes Llama-4 Maverick, Gemma-2 denotes Gemma-2-27B-IT, and Qwen2.5 denotes Qwen2.5-72B-Instruct.}}
\label{tab:newer_masking}
\end{table*}

\camerarevise{Masking lowers accuracy in 22 of 24 comparisons, leaves one unchanged, and improves one, with an average reduction of 20.14 percentage points. The pattern on post-cutoff and synthetic datasets reduces the likelihood that the effect is explained only by memorization of old benchmark records. Nevertheless, masking changes the semantic information available to the model and therefore cannot by itself prove the absence of pretraining memorization.}

\subsection{\revise{Experimental Details for Observation 2}}
\label{app:obs2_exp}

\begin{camerarevision}
We compare rule sets and path sets with the same total number of atomic feature conditions. Under a condition budget $n$, a rule set contains $n$ independent single-condition rules, whereas a path set contains one or more conjunctive paths with $n$ conditions in total. The two prompts use identical features, labels, and comparison operators. We use Qwen2.5-72B with temperature 0.8 and evaluate condition budgets $n\in\{3,4\}$.

Let $\{(x_i,y_i)\}_{i=1}^{N}$ be the evaluation set and $S$ a generated rule or path set. A member $s\in S$ is triggered when $x_i$ satisfies its condition or all conditions along its path, and predicts label $\hat y_s$. We define
\begin{align*}
T_i(S)&=\{s\in S:s\text{ is triggered by }x_i\},\\
C_i(S)&=\{s\in T_i(S):\hat y_s=y_i\}.
\end{align*}
Let $t_i=|T_i(S)|$ and $c_i=|C_i(S)|$. Set Utility and Set Accuracy are computed as
\[
\operatorname{Utility}(S)=\frac{1}{N}\sum_{i=1}^{N}\mathbf{1}[c_i>0],
\]
and
\begin{align*}
\operatorname{Accuracy}(S)&=\frac{1}{N}\sum_{i=1}^{N}a_i(S),\\
a_i(S)&=
\begin{cases}
c_i/t_i, & t_i>0,\\
0, & t_i=0.
\end{cases}
\end{align*}
Set Utility measures whether an instance is covered by at least one correct member, whereas Set Accuracy measures the proportion of correct predictions among triggered members and assigns zero when none is triggered. We average both metrics over 100 generations for each dataset, set type, and condition budget.
\end{camerarevision}

\begin{table*}[t]
\centering
\begin{camerarevision}
\begin{tabular*}{\appendixfulltablewidth}{@{\extracolsep{\fill}}lrrrrrrrr@{}}
\toprule
& \multicolumn{4}{c}{Condition budget $n=3$}
& \multicolumn{4}{c}{Condition budget $n=4$} \\
\cmidrule(lr){2-5}\cmidrule(lr){6-9}
Dataset & Rule U. & Path U. & Rule A. & Path A.
        & Rule U. & Path U. & Rule A. & Path A. \\
\midrule
Nursery    & 31.87 & 28.32 & 24.75 & 25.86 & 41.34 & 41.35 & 31.20 & 33.32 \\
Diabetes   & 35.28 & 32.21 & 32.60 & 29.48 & 64.52 & 65.14 & 50.62 & 50.29 \\
Spambase   & 54.55 & 41.06 & 45.44 & 36.64 & 60.42 & 53.63 & 52.04 & 47.74 \\
Abalone    & 33.41 & 29.68 & 33.05 & 29.45 & 53.11 & 32.23 & 51.42 & 32.05 \\
Blood      & 40.01 & 35.45 & 33.92 & 32.03 & 64.24 & 31.66 & 43.48 & 24.49 \\
Iris       & 74.73 & 60.81 & 65.18 & 54.84 & 73.51 & 69.10 & 63.72 & 63.45 \\
Breast     & 38.99 & 36.28 & 36.77 & 35.66 & 51.82 & 18.39 & 44.57 & 18.12 \\
Glioma     & 65.94 & 65.88 & 60.76 & 62.94 & 75.96 & 50.17 & 68.29 & 46.43 \\
\midrule
Mean & \textbf{46.85} & 41.21 & \textbf{41.56} & 38.36
     & \textbf{60.62} & 45.21 & \textbf{50.67} & 39.49 \\
\bottomrule
\end{tabular*}
\end{camerarevision}
\caption{\camerarevise{Rule-set and path-set results under matched condition budgets (\%). U. and A. denote Set Utility and Set Accuracy. Each entry is averaged over 100 generations.}}
\label{tab:controlled_rule_path}
\end{table*}

\camerarevise{Table~\ref{tab:controlled_rule_path} shows that rule sets achieve higher mean Set Utility and Set Accuracy than path sets under both budgets, indicating that the advantage persists when the total number of atomic conditions is controlled.}\subsection{\revise{Dataset-Specific Parameters}}
\label{app:dataset_params}

Table~\ref{tab:dataset_params} lists dataset-specific parameters. Max Depth (3 or 4, depending on dataset) is shared by all methods. Test Size is 30 for Breast dataset and 100 for others. The n\_estimators column applies to XGBoost, Random Forest and LLMT Forest.

\begin{table}[h]
\centering
\begin{tabular*}{\appendixsinglecoltablewidth}{@{\extracolsep{\fill}}lcc@{}}
\toprule
\tb{Dataset} & \tb{Max Depth}  & \tb{n\_estimators} \\ 
\midrule

Diabetes & 3  & 3 \\
Nursery & 3  & 3 \\
Spambase & 4  & 6 \\
Abalone & 3  & 3 \\
Blood & 3  & 3 \\
Iris & 3  & 3 \\
Breast & 3  & 3 \\
Glioma & 4  & 6 \\
\revised{Communities} & \revised{4}  & \revised{30} \\
\revised{Ecom} & \revised{3}  & \revised{5} \\
\revised{Myocardial} & \revised{3}  & \revised{16} \\

\bottomrule
\end{tabular*}
\caption{Dataset-specific parameters for all experiments.}
\label{tab:dataset_params}
\end{table}
\subsection{\revise{LLM Configuration Settings}}
\label{app:llm_config}

Table~\ref{tab:llm_config} details the LLM API configuration used for all LLM-based methods.
\camerarevise{The same Qwen2.5-72B backbone is used for all baselines. Request intervals and timeouts are operational parameters adjusted to dataset workload and provider status rather than fixed experimental hyperparameters.}

\begin{table}[t]
\centering
\begin{tabularx}{\appendixsinglecoltablewidth}{@{}lX@{}}
\toprule
\tb{Parameter} & \tb{Value} \\ \midrule

Model
& Qwen2.5-72B-Instruct \\

API Provider
& TogetherAI \\

Temperature
& 0.0 \\

Max Tokens
& 2048\\

Parallel Batch Size
& 6 \\

Presence Penalty
& 0.0 \\

Frequency Penalty
& 0.0 \\

Random Seed
& 42 \\

\bottomrule
\end{tabularx}
\caption{LLM API configuration settings used for all LLM-based methods.}
\label{tab:llm_config}
\end{table}
\subsection{\revise{Method-Specific Settings}}
\label{app:method_params}

Table~\ref{tab:method_params} provides hyperparameters for all methods. The max\_depth settings match those in Table~\ref{tab:dataset_params} for each dataset. The n\_estimators values are dataset-specific and listed in Table~\ref{tab:dataset_params}. All LLM-based methods (IO-Tree, CoT-Tree, ToT-Tree, LLMT) use the LLM configuration specified in Table~\ref{tab:llm_config}. It is worth noting for LLMT and LLMT Forest, confidence groups are formed by grouping meta-rules with the same confidence value together.

\begin{table*}[t]
\centering
\fontsize{9.5}{10.5}\selectfont
\renewcommand{\arraystretch}{1.02}

\begin{minipage}{\appendixfulltablewidth}
\begin{minipage}[t]{0.485\linewidth}
\vspace{0pt}
\begin{tabularx}{\linewidth}{@{}
    >{\raggedright\arraybackslash}p{0.24\linewidth}
    >{\raggedright\arraybackslash}p{0.38\linewidth}
    >{\raggedright\arraybackslash}X@{}}
\toprule
\tb{Method} & \tb{Parameter} & \tb{Value} \\
\midrule

Tree methods
& Maximum depth
& Table~\ref{tab:dataset_params} \\

Tree ensembles
& Ensemble size
& Table~\ref{tab:dataset_params} \\
\midrule

Step-ZSDT
& Leaf-stopping threshold $\tau$
& 0.9 \\
\midrule

XGBoost
& \texttt{learning\_rate}
& 0.1 \\
\midrule

LogReg
& \texttt{max\_iter}
& 1000 \\
\midrule

Linear SVM
& Estimator
& \texttt{LinearSVC} \\
& \texttt{max\_iter}
& 10000 \\
\midrule

LightGBM
& \texttt{learning\_rate}
& 0.1 \\
\midrule

CatBoost
& \texttt{depth}
& Table~\ref{tab:dataset_params} \\
& \texttt{learning\_rate}
& 0.1 \\
\midrule

ToT-Tree
& Candidate rules per node
& 5 \\
& Voting rounds
& 5 \\
\bottomrule
\end{tabularx}
\end{minipage}\hfill%
\begin{minipage}[t]{0.485\linewidth}
\vspace{0pt}
\begin{tabularx}{\linewidth}{@{}
    >{\raggedright\arraybackslash}p{0.20\linewidth}
    >{\raggedright\arraybackslash}p{0.34\linewidth}
    >{\raggedright\arraybackslash}X@{}}
\toprule
\tb{Method} & \tb{Parameter} & \tb{Value} \\
\midrule

FeatLLM
& Rule-generation trials
& 5 \\
\midrule

DeLTa
& Residual learner
& Decision-tree regressor \\
& Fusion coefficient $\eta$
& 0.1 \\
\midrule

GPTree
& Node type
& \textsc{code} (our reproduction) \\
\midrule

TabPFN
& Pretraining-limit override
& Enabled \\
\midrule

TabLLM
& Backbone
& \texttt{bigscience/T0pp} \\
& Adaptation
& IA3 via T-Few \\
\midrule

LLMT
& Leaf-refinement threshold $\tau$
& 0.7 \\
& Numeric split candidates per feature
& 10 \\
& Requested meta-rules $K$
& $\max(10,2^{d_{\max}}-1)$ \\
\midrule

LLMT Forest
& Base estimator
& LLMT (same settings) \\
& Feature partition
& Randomized, balanced, and disjoint \\
& Aggregation
& Majority vote \\

\bottomrule
\end{tabularx}
\end{minipage}
\end{minipage}

\caption{\camerarevise{Method-specific settings.}}
\label{tab:method_params}
\end{table*}
\subsection{\revise{Experimental Running Parameters}}
\label{app:running_params}

Table~\ref{tab:running_params} lists general experimental running parameters. Number of Trials per Setting means that for each training sample size, we randomly sample 10 different training sets and run experiments, i.e., build 10 trees.

\begin{table}[t]
\centering
\begin{tabularx}{\appendixsinglecoltablewidth}{@{}lX@{}}
\toprule
\tb{Parameter} & \tb{Value} \\
\midrule
Number of Trials per Setting & 10 \\
Random Seed & 0 \\
Train Batch Size & 8 \\
Test Batch Size & 8 \\
\bottomrule
\end{tabularx}
\caption{General experimental running parameters.}
\label{tab:running_params}
\end{table}

\subsection{\revise{Implementation of two Zero-shot Baselines}}

To further evaluate the effectiveness and efficiency of our proposed LLMT framework, we compare it with two representative zero-shot LLM-based decision tree baselines: (1) Direct-ZSDT (\citet{knauer2024oh}); (2) Step-ZSDT (\citet{carrasco2025zero}). For baseline Direct-ZSDT, we reuse the official code and ensure identical prompt and parameter settings to guarantee fair comparison. For baseline Step-ZSDT, whose code is not released, we follow the prompt templates and methodology described by the authors and set the probability threshold to 0.9 as specified in the original paper.

\section{Additional Experimental Results}
\label{app:additional_results}

\subsection{Accuracy}

The mean accuracy and standard deviation across different \#shots are presented in Figure~\ref{fig:accuracy_three}.

\begin{figure*}[t]
    \centering
    \includegraphics[width=\linewidth]{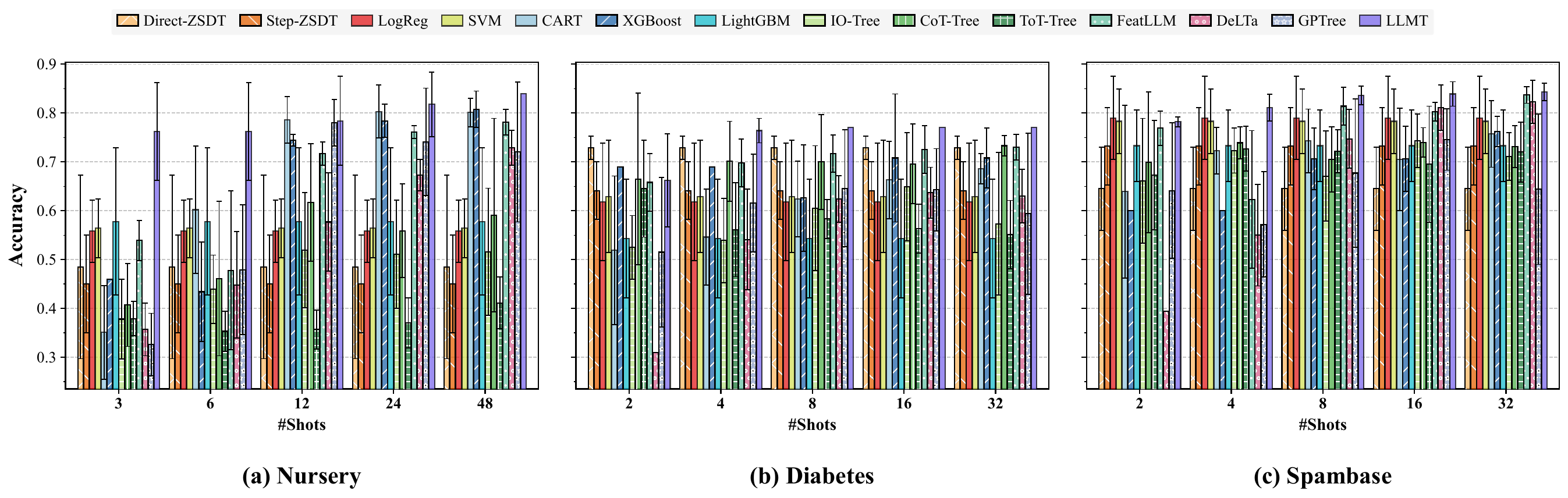}
    \caption{\newrevise{The accuracy and standard deviation of different methods on three datasets. }}
    \label{fig:accuracy_three}
\end{figure*}

\subsection{\camerarevise{Comparisons with Additional Baselines}}

\camerarevise{Tables~\ref{tab:conventional_fixed}, \ref{tab:conventional_shots}, and~\ref{tab:blackbox_results} present expanded comparisons with tree-ensemble and pretrained black-box baselines. }

\begin{table*}[t]
\centering
\begin{camerarevision}
\begin{tabular*}{\appendixfulltablewidth}{@{\extracolsep{\fill}}lrrr@{}}
\toprule
Dataset & Random Forest & CatBoost & LLMT \\
\midrule
Nursery & 0.397 & 0.473 & \textbf{0.762} \\
Diabetes & 0.584 & 0.590 & \textbf{0.764} \\
Spambase & 0.725 & 0.743 & \textbf{0.811} \\
Abalone & 0.688 & \textbf{0.700} & \underline{0.694} \\
Blood & 0.549 & 0.558 & \textbf{0.675} \\
Iris & 0.769 & 0.829 & \textbf{0.940} \\
Breast & 0.530 & 0.517 & \textbf{0.617} \\
Glioma & 0.733 & 0.759 & \textbf{0.760} \\
\midrule
Average & 0.622 & 0.646 & \textbf{0.753} \\
\bottomrule
\end{tabular*}
\end{camerarevision}
\caption{\camerarevise{Accuracy of Random Forest and CatBoost with two training examples per class, compared with LLMT.}}
\label{tab:conventional_fixed}
\end{table*}

\begin{table*}[t]
\centering
\begin{camerarevision}
\begin{tabular*}{\appendixfulltablewidth}{@{\extracolsep{\fill}}llrrr@{}}
\toprule
Dataset & \#Shots & Random Forest & CatBoost & LLMT \\
\midrule
\multirow{5}{*}{Nursery}
& 3  & 0.349 & 0.464 & \textbf{0.762} \\
& 6  & 0.397 & 0.473 & \textbf{0.762} \\
& 12 & 0.707 & 0.411 & \textbf{0.784} \\
& 24 & 0.754 & 0.721 & \textbf{0.818} \\
& 48 & 0.798 & 0.740 & \textbf{0.840} \\
\midrule
\multirow{5}{*}{Diabetes}
& 2  & 0.486 & 0.523 & \textbf{0.662} \\
& 4  & 0.584 & 0.590 & \textbf{0.764} \\
& 8  & 0.659 & 0.650 & \textbf{0.770} \\
& 16 & 0.618 & 0.671 & \textbf{0.770} \\
& 32 & 0.683 & 0.700 & \textbf{0.770} \\
\midrule
\multirow{5}{*}{Spambase}
& 2  & 0.677 & 0.620 & \textbf{0.782} \\
& 4  & 0.725 & 0.743 & \textbf{0.811} \\
& 8  & 0.740 & 0.785 & \textbf{0.836} \\
& 16 & 0.783 & 0.788 & \textbf{0.839} \\
& 32 & 0.821 & 0.785 & \textbf{0.843} \\
\bottomrule
\end{tabular*}
\end{camerarevision}
\caption{\camerarevise{Accuracy of Random Forest and CatBoost across shot counts, compared with LLMT.}}
\label{tab:conventional_shots}
\end{table*}

\begin{table*}[t]
\centering
\begin{camerarevision}
\begin{tabular*}{\appendixfulltablewidth}{@{\extracolsep{\fill}}lrrrr@{}}
\toprule
Dataset & TabPFN & TabLLM & LLMT & LLMT Forest \\
\midrule
Nursery & 0.750 & 0.759 & \textbf{0.762} & -- \\
Diabetes & 0.710 & 0.510 & \textbf{0.764} & -- \\
Spambase & 0.802 & 0.673 & \textbf{0.811} & -- \\
Abalone & \textbf{0.754} & 0.569 & \underline{0.694} & -- \\
Blood & 0.658 & 0.496 & \textbf{0.675} & -- \\
Iris & 0.918 & 0.327 & \textbf{0.940} & -- \\
Breast & 0.510 & 0.587 & \textbf{0.617} & -- \\
Glioma & \textbf{0.878} & 0.521 & \underline{0.760} & -- \\
\midrule
Communities & 0.443 & N/A & 0.438 & \textbf{0.490} \\
Ecom & 0.518 & N/A & \underline{0.582} & \textbf{0.757} \\
Myocardial & 0.541 & N/A & \underline{0.618} & \textbf{0.773} \\
\bottomrule
\end{tabular*}
\end{camerarevision}
\caption{\camerarevise{Comparison with pretrained black-box few-shot baselines. TabLLM runs out of memory on Communities and Ecom and is prohibitively slow on Myocardial under the T0pp protocol; these cases are marked N/A.}}
\label{tab:blackbox_results}
\end{table*}

\subsection{Results on High-Dimensional Datasets}
\label{app:high_dimensional}

\camerarevise{Non-ensemble and ensemble methods are reported in Tables~\ref{tab:additional_dataset_results} and~\ref{tab:forest_results}, respectively. Step-ZSDT is omitted because its per-feature, per-node LLM queries are prohibitively expensive in high dimensions. Overall, LLMT performs best among non-ensemble methods, achieving the highest average accuracy. LLMT Forest achieves the highest average accuracy among ensemble methods.}

\begin{table*}[!t]
\centering
\fontsize{9.0}{9.8}\selectfont
\setlength{\tabcolsep}{3pt}
\renewcommand{\arraystretch}{1.2}
\begin{tabular*}{\appendixfulltablewidth}{@{}l@{\extracolsep{\fill}}*{10}{c}@{}}
\toprule
\textbf{Dataset}
& \revised{\textbf{LogReg}}
& \revised{\textbf{SVM}}
& \textbf{CART}
& \textbf{Direct-ZSDT}
& \textbf{IO-Tree}
& \textbf{CoT-Tree}
& \textbf{ToT-Tree}
& \revised{\textbf{FeatLLM}}
& \revised{\textbf{GPTree}}
& \textbf{LLMT (Ours)} \\
\midrule

Communities
& 0.424
& \revised{\tb{0.439}}
& 0.404
& 0.418
& 0.336
& 0.411
& 0.402
& \revised{0.432}
& \revised{0.407}
& \underline{0.438} \\

Ecom
& 0.509
& 0.508
& 0.554
& 0.561
& 0.511
& 0.524
& 0.534
& \revised{0.579}
& \revised{0.490}
& \tb{0.582} \\

Myocardial
& 0.526
& 0.522
& 0.540
& 0.510
& 0.535
& 0.503
& 0.498
& \revised{0.596}
& \revised{0.556}
& \tb{0.618} \\

\midrule
\textbf{Average}
& 0.486
& 0.490
& 0.499
& 0.496
& 0.461
& 0.479
& 0.478
& 0.536
& 0.484
& \tb{0.546} \\

\bottomrule
\end{tabular*}
\caption{\camerarevise{Accuracy comparison of non-ensemble methods on three high-dimensional datasets.}}
\label{tab:additional_dataset_results}
\end{table*}

\begin{table*}[!t]
\centering
\fontsize{9.5}{10.4}\selectfont
\setlength{\tabcolsep}{2pt}
\renewcommand{\arraystretch}{1.2}
\begin{tabular*}{\appendixfulltablewidth}{@{}l@{\extracolsep{\fill}}*{6}{c}@{}}
\toprule
\textbf{Dataset}
& \textbf{XGBoost}
& \textbf{RandomForest}
& \textbf{LightGBM}
& \textbf{CatBoost}
& \revised{\textbf{DeLTa}}
& \textbf{LLMT Forest (Ours)} \\
\midrule

Communities
& 0.428
& 0.440
& 0.434
& 0.437
& \revised{0.473}
& \tb{0.490} \\

Ecom
& 0.695
& 0.504
& 0.412
& 0.503
& \revised{0.619}
& \tb{0.757} \\

Myocardial
& 0.702
& 0.545
& 0.581
& 0.535
& \revised{0.565}
& \tb{0.773} \\

\midrule
\textbf{Average}
& 0.608
& 0.496
& 0.476
& 0.492
& 0.552
& \tb{0.673} \\

\bottomrule
\end{tabular*}
\caption{\camerarevise{Accuracy comparison of ensemble methods on three high-dimensional datasets.}}
\label{tab:forest_results}
\end{table*}

\subsection{\revise{Interpretability Analyses}}

\subsubsection{\revise{Rule Diversity Analysis}}

We ensure rule diversity through careful prompt design. Our meta-rule generation prompt (see Figures~\ref{fig:generation_prompt} and~\ref{fig:generation_prompt_cont} for prompt details) encourages diversity through the following constraints:
\begin{itemize}
    \item \textbf{No Redundancy}
    \item \textbf{Maximize Purity}
\end{itemize}

\camerarevise{To quantify feature coverage, we generate 10 meta-rule sets for each of the eight original datasets and count the distinct features used. We use the same rule-pool size $K$ as in the main experiments (Table~\ref{tab:method_params}). Table~\ref{tab:rule_diversity} reports both the number and proportion of covered features.}

\begin{table*}[h]
    \centering
    \resizebox{\appendixfulltablewidth}{!}{%
    \begin{tabular}{@{}lccccccccc@{}}
    \toprule
    \tb{Dataset} & \tb{Diabetes} & \tb{Iris} & \tb{Spambase} & \tb{Nursery} & \tb{Abalone} & \tb{Blood} & \tb{Breast} & \tb{Glioma} \\ \midrule
    \#Rules ($K$) & 10 & 10 & 15 & 10 & 10 & 10 & 10 & 15 \\
    \#Features & 8 & 4 & 18 & 8 & 8 & 4 & 9 & 23 \\ 
    Avg. Features Used & 8.0 & 4.0 & 15.0 & 8.0 & 8.0 & 4.0 & 9.0 & 14.9 \\
    Feature Coverage & 100\% & 100\% & 83.3\% & 100\% & 100\% & 100\%  & 100\% & 64.8\% \\ \bottomrule
    \end{tabular}
    }
    \caption{Rule diversity analysis: Feature coverage (Avg. Features Used / \#Features) measures the proportion of features utilized in the generated rules, with higher values indicating better diversity.}
    \label{tab:rule_diversity}
\end{table*}

\camerarevise{Coverage is 100\% on six datasets, 83.3\% on Spambase, and 64.8\% on the 23-feature Glioma dataset. Coverage is not itself an interpretability score---a smaller, focused rule set may be easier to inspect---but these values show that rule generation is not restricted to a single dominant feature.}

\subsubsection{\revise{Visualized Tree Examples}}

\camerarevise{Figures~\ref{fig:direct-zsdt-tree}--\ref{fig:llmt-tree} provide representative Diabetes trees from every compared tree generator. The LLMT example uses glucose as its principal decision feature and exposes every threshold and leaf prediction directly. These examples complement the aggregate impurity and tree-size results in Section~\ref{sec:exp} by allowing readers to inspect the learned decision logic, while we avoid treating a single medically plausible example as quantitative evidence of clinical validity.}

\begin{figure}[!]
    \centering
    \includegraphics[width=0.82\columnwidth]{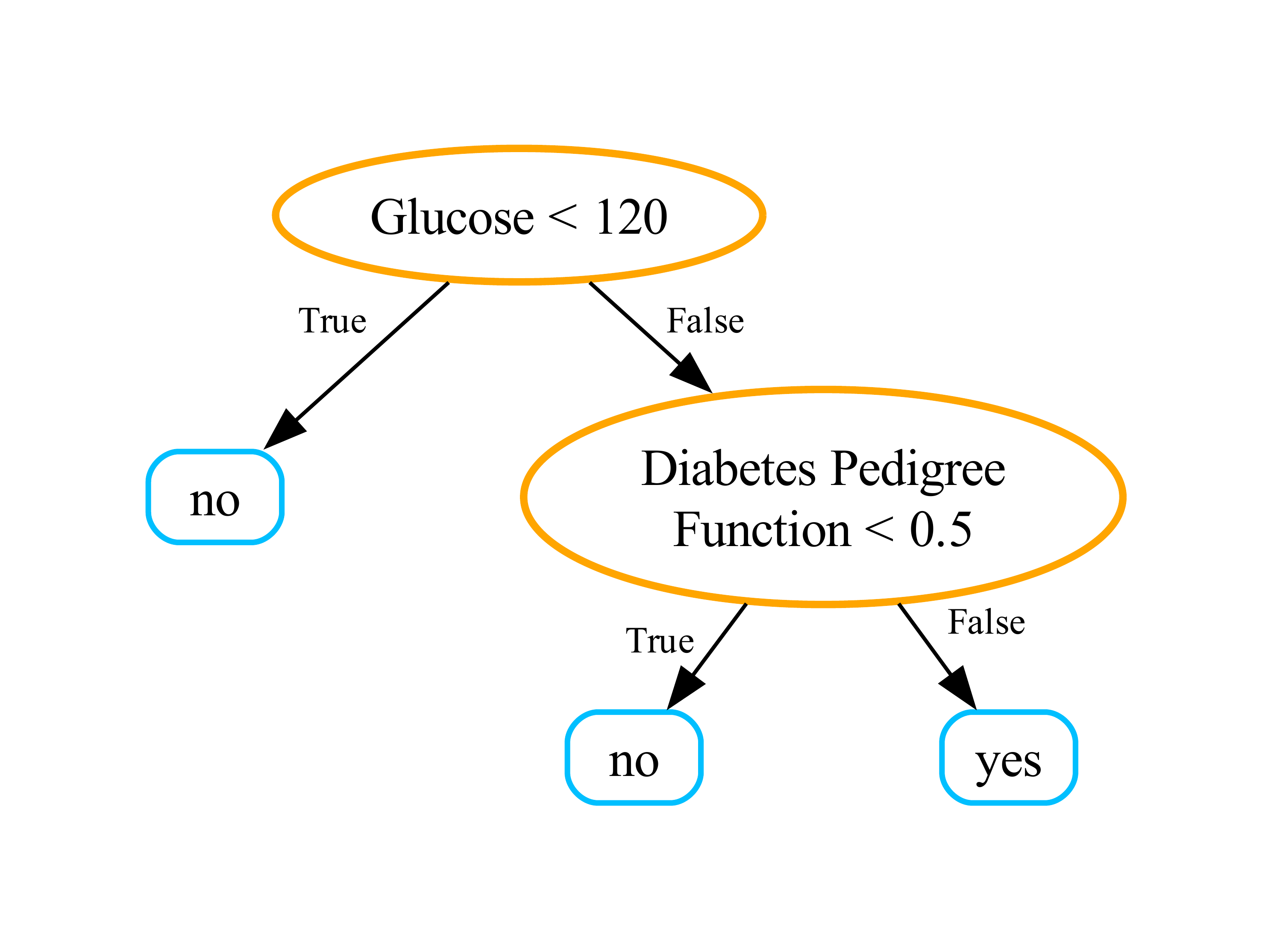}
    \caption{\newrevise{An Example of Direct-ZSDT}}
    \label{fig:direct-zsdt-tree}
\end{figure}

\begin{figure}[!]
    \centering
    \includegraphics[width=0.82\columnwidth]{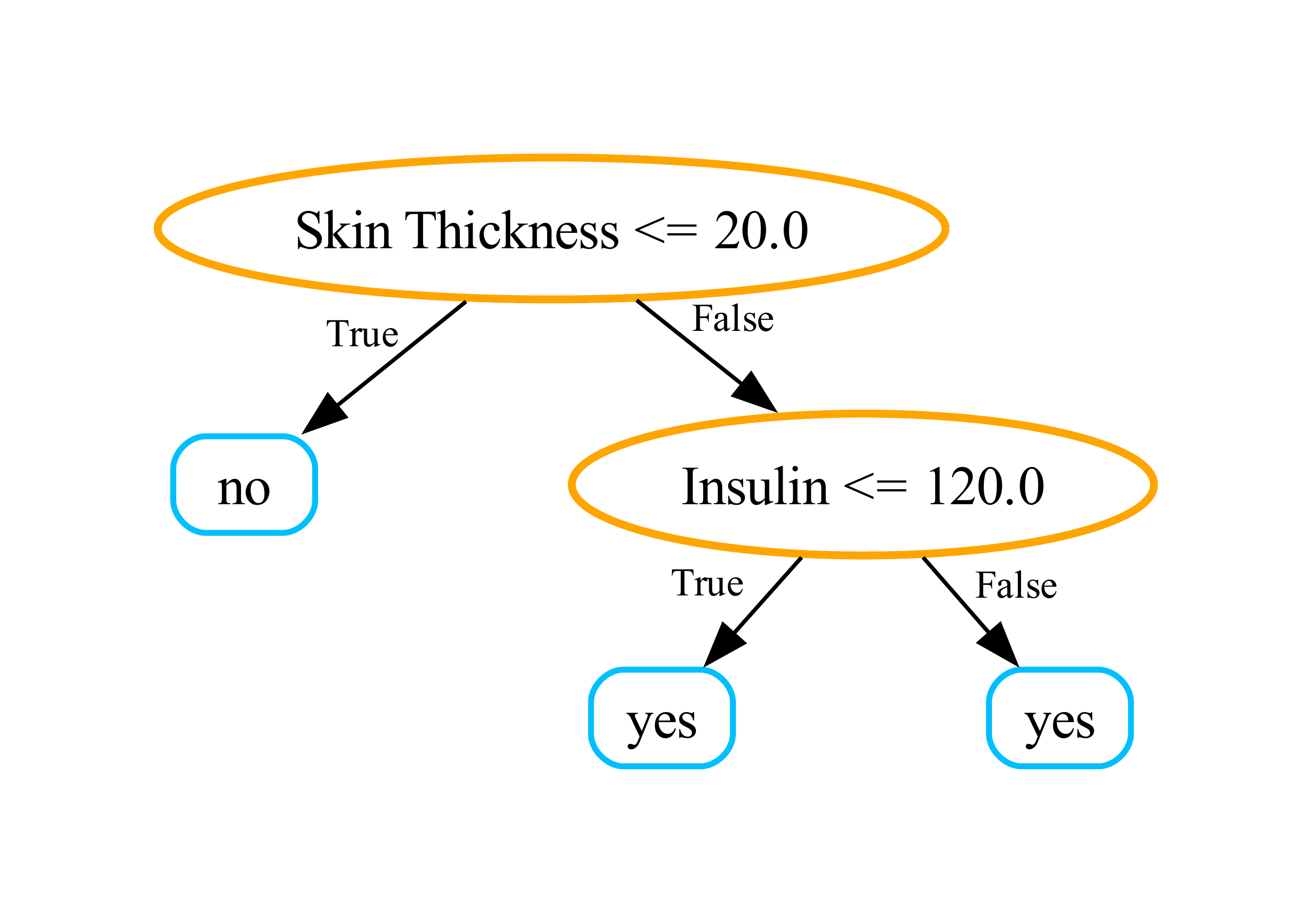}
    \caption{\newrevise{An Example of Step-ZSDT}}
    \label{fig:step-zsdt-tree}
\end{figure}

\begin{figure}[!]
    \centering
    \includegraphics[width=0.45\columnwidth]{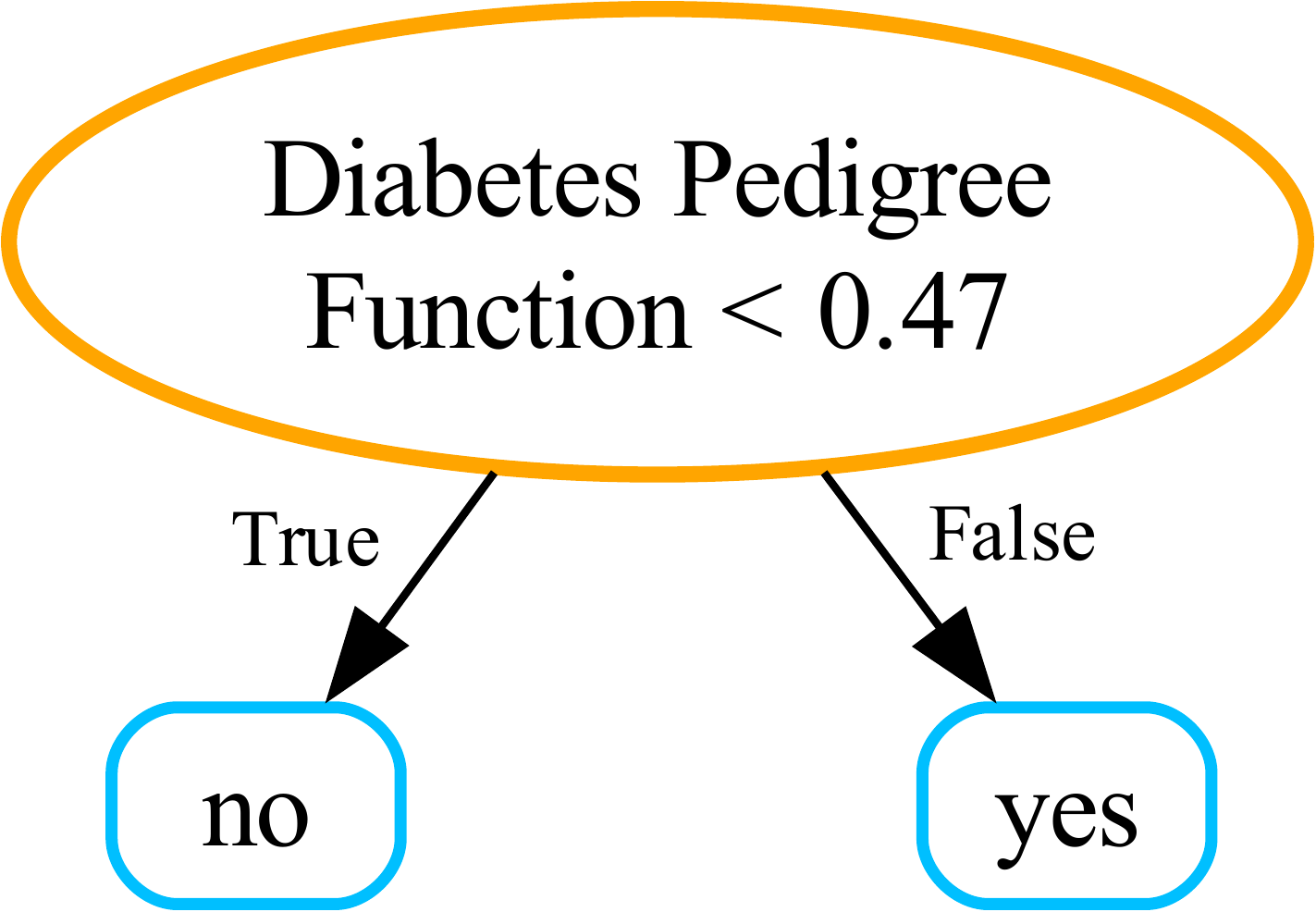}
    \caption{An Example of CART Tree}
    \label{fig:cart-tree}
\end{figure}

\begin{figure}[!]
    \centering
    \includegraphics[width=0.55\columnwidth]{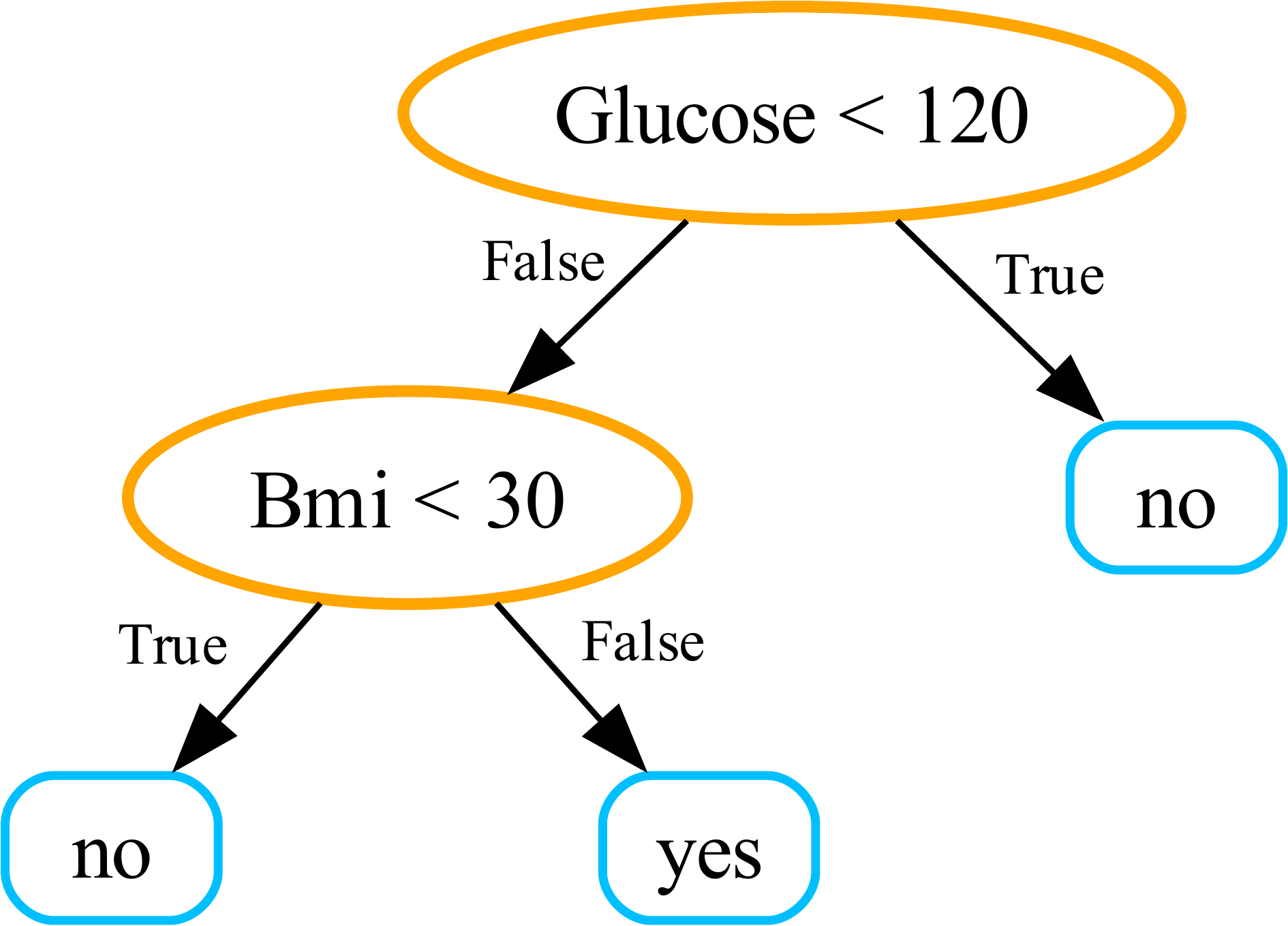}
    \caption{An Example of IO-Tree}
    \label{fig:io-tree}
\end{figure}

\begin{figure}[!]
    \centering
    \includegraphics[width=0.48\columnwidth]{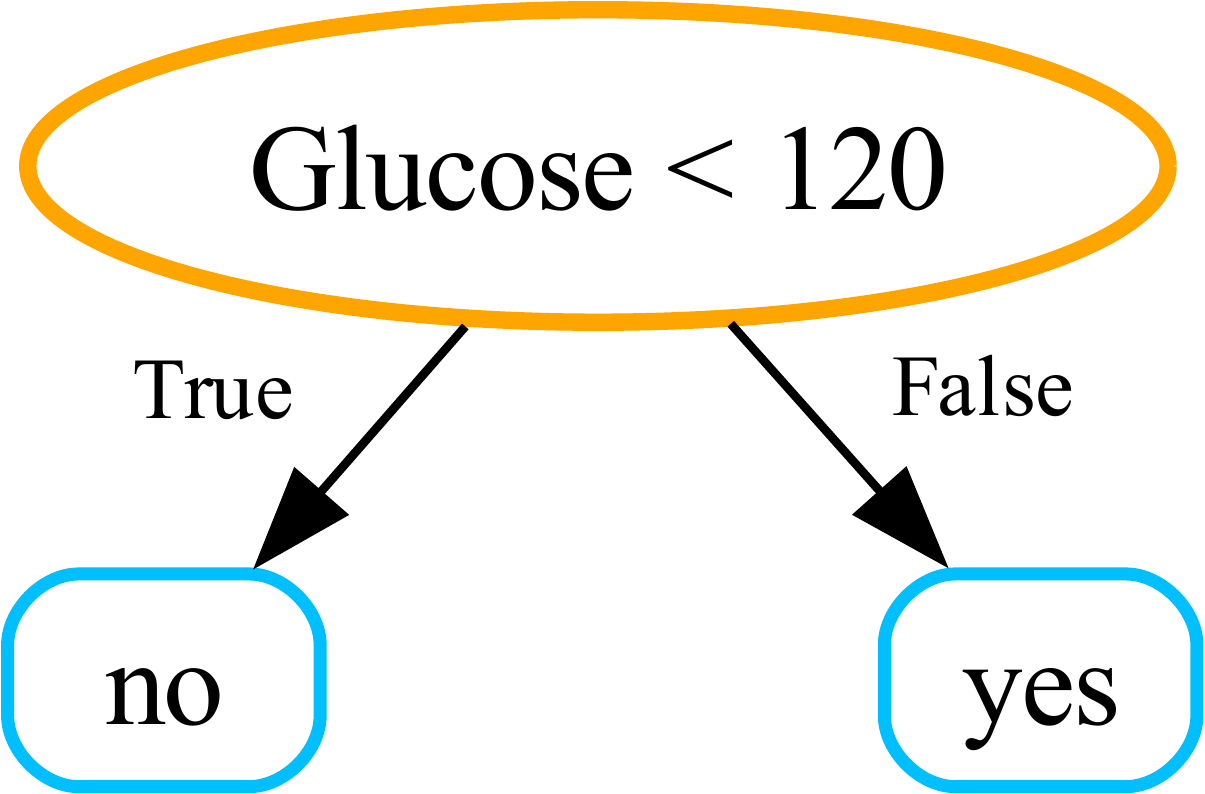}
    \caption{An Example of CoT-Tree}
    \label{fig:cot-tree}
\end{figure}

\begin{figure}[!]
    \centering
    \includegraphics[width=0.90\columnwidth]{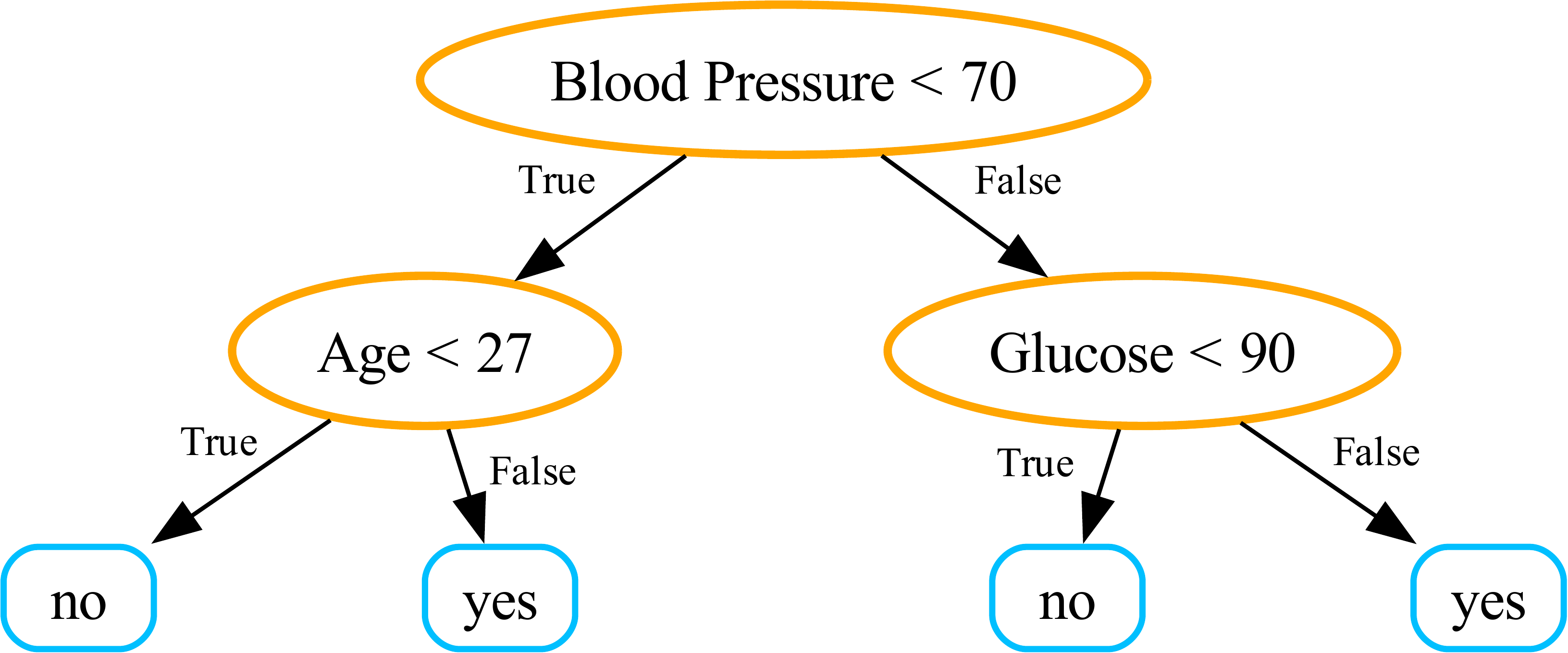}
    \caption{An Example of ToT-Tree}
    \label{fig:tot-tree}
\end{figure}

\begin{figure}[!]
    \centering
    \includegraphics[width=0.82\columnwidth]{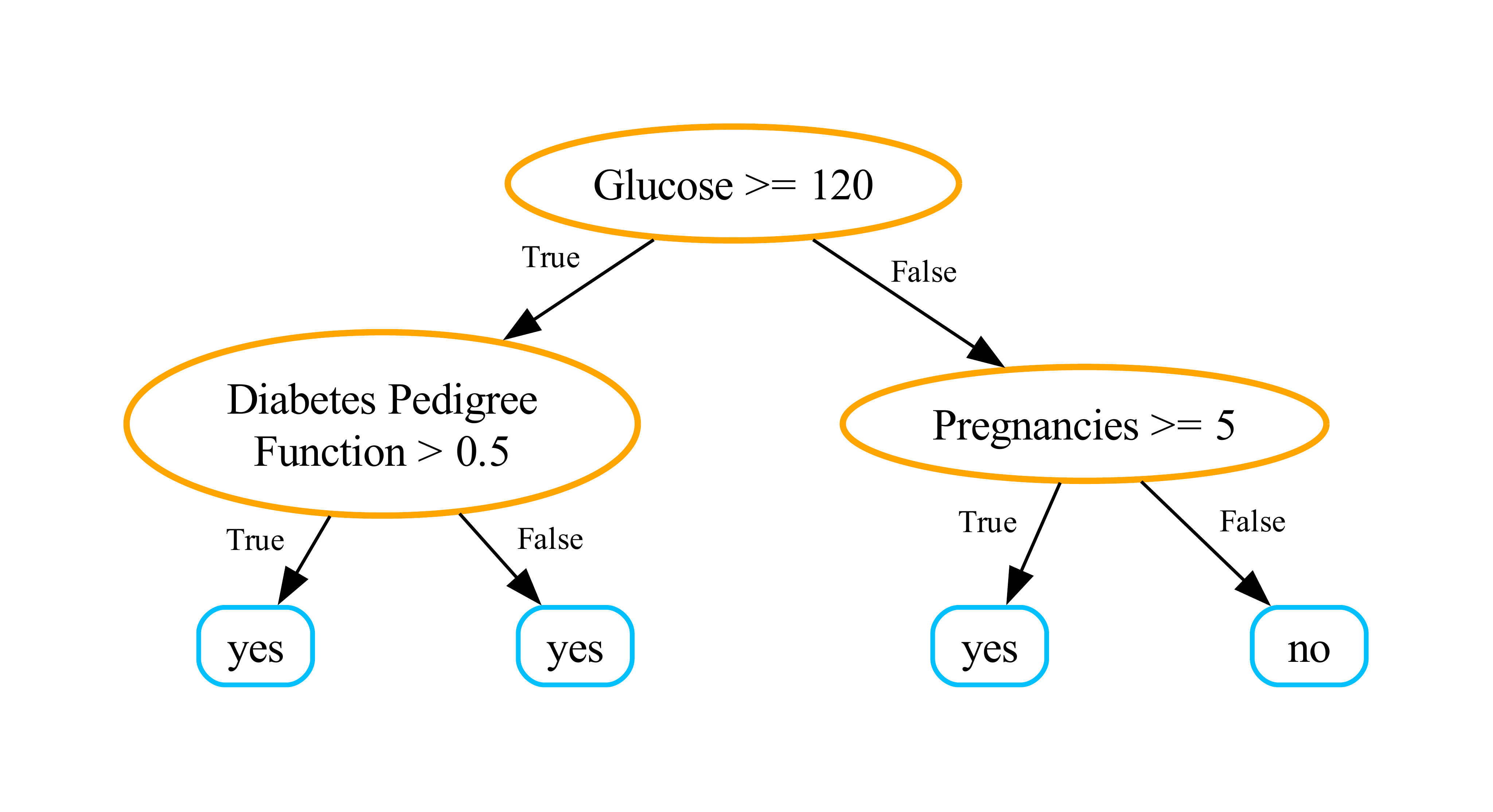}
    \caption{\newrevise{An Example of GPTree}}
    \label{fig:gptree-tree}
\end{figure}

\begin{figure}[!]
    \centering
    \includegraphics[width=0.45\columnwidth]{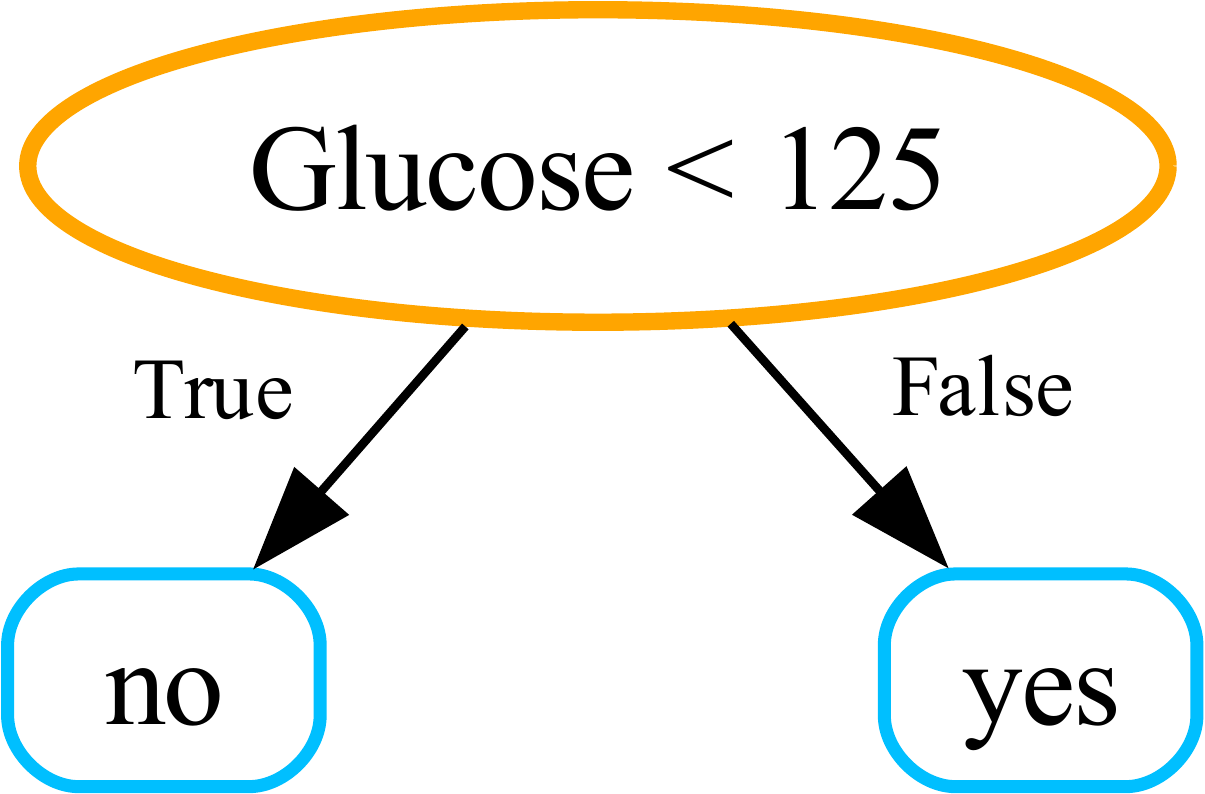}
    \caption{An Example of LLMT}
    \label{fig:llmt-tree}
\end{figure}

% \newpage
\subsection{Sensitivity to Maximum Tree Depth}
\label{sec:depth}
During parameter tuning, we examine the effect of varying maximum depth in the tree‐building process; results are shown in Table~\ref{tab:depth}. Performance proves sensitive to this hyperparameter. On low‐dimensional datasets (e.g., Diabetes), even shallow trees (depth = 2-3) yield stable accuracy. By contrast, on higher‐dimensional datasets with numerous categorical features (e.g., Nursery), deeper trees (depth = 4) consistently achieve better splits and higher accuracy. Therefore, careful tuning of maximum depth is essential - deeper trees better capture complex feature interactions when feature diversity is high.

\begin{table*}[!]
\centering
% \resizebox{0.95\textwidth}{!}{% 缩小5%的缩放比例
\begin{tabular*}{\appendixfulltablewidth}{@{\extracolsep{\fill}}llcccc@{}}
\toprule
\tb{Datasets} & \tb{\#Shots} & \tb{2} & \tb{3} & \tb{4} & \tb{5} \\ \midrule

\multirow{5}{*}{Nursery}    
& 3  & 0.740$\pm$0.000 & 0.762$\pm$0.100 & 0.762$\pm$0.100 & 0.762$\pm$0.100 \\ 
& 6  & 0.740$\pm$0.000 & 0.762$\pm$0.100 & 0.707$\pm$0.063 & 0.731$\pm$0.077 \\ 
& 12 & 0.740$\pm$0.000 & 0.784$\pm$0.091 & 0.742$\pm$0.074 & 0.737$\pm$0.054 \\ 
& 24 & 0.740$\pm$0.000 & 0.818$\pm$0.066 & 0.713$\pm$0.084 & 0.735$\pm$0.045 \\ 
& 48 & 0.740$\pm$0.000 & 0.840$\pm$0.000 & 0.788$\pm$0.042 & 0.742$\pm$0.039 \\ \midrule

\multirow{5}{*}{Diabetes}     
& 2  & 0.662$\pm$0.095 & 0.662$\pm$0.095 & 0.662$\pm$0.095 & 0.669$\pm$0.083 \\ 
& 4  & 0.710$\pm$0.092 & 0.764$\pm$0.025 & 0.736$\pm$0.064 & 0.736$\pm$0.064 \\ 
& 8  & 0.770$\pm$0.000 & 0.770$\pm$0.000 & 0.763$\pm$0.016 & 0.758$\pm$0.020 \\ 
& 16 & 0.770$\pm$0.000 & 0.770$\pm$0.000 & 0.763$\pm$0.016 & 0.756$\pm$0.024 \\ 
& 32 & 0.770$\pm$0.000 & 0.770$\pm$0.000 & 0.766$\pm$0.020 & 0.762$\pm$0.015 \\ \midrule

\multirow{5}{*}{Spambase}  
& 2  & 0.782$\pm$0.010 & 0.782$\pm$0.010 & 0.782$\pm$0.010 & 0.782$\pm$0.010 \\ 
& 4  & 0.788$\pm$0.006 & 0.811$\pm$0.027 & 0.811$\pm$0.027 & 0.784$\pm$0.028 \\ 
& 8  & 0.788$\pm$0.006 & 0.843$\pm$0.028 & 0.836$\pm$0.019 & 0.810$\pm$0.023 \\ 
& 16 & 0.790$\pm$0.000 & 0.825$\pm$0.035 & 0.839$\pm$0.025 & 0.810$\pm$0.018 \\ 
& 32 & 0.790$\pm$0.000 & 0.811$\pm$0.032 & 0.843$\pm$0.018 & 0.798$\pm$0.016 \\ 
\bottomrule
\end{tabular*}%
\caption{Experimental results across different maximum tree depths.}
\label{tab:depth}
% }
\end{table*}

\section{Usage of LLMs}
In this work, we use ChatGPT to polish the writing of our paper.

\section{Prompting Templates}
\label{app:prompt}
% The examples of three table-to-text serialization templates are presented in Figure~\ref{fig:table2text}.
To facilitate application across diverse tasks, we design two prompt templates: one for generating meta‐rule candidates (Figure~\ref{fig:generation_prompt} and Figure~\ref{fig:generation_prompt_cont}) and another for producing supervised probability outputs for leaf labels (Figure~\ref{fig:refinement_prompt}). We also include Figure \ref{fig:io_prompt}, \ref{fig:io_prompt_cont}, \ref{fig:cot_prompt}, \ref{fig:cot_prompt_cont}, \ref{fig:tot_prompt}, and \ref{fig:tot_prompt_cont}, which respectively show the reference I/O-Tree, CoT-Tree, and ToT-Tree generation templates used as baselines in the main text. Text in \textcolor{blue}{blue} font denotes the title of the prompt section, \textcolor{red}{red} font signifies variables that vary across datasets, \textcolor{orange}{orange} font represents system commands, and text in \tb{black} font is the fixed input text used for prompting.

% \enlargethispage{100cm} % 扩展1cm高度

% \enlargethispage{100cm} % 扩展1cm高度

\begin{figure*}[t]
    \centering
    \includegraphics[width=0.8\linewidth]{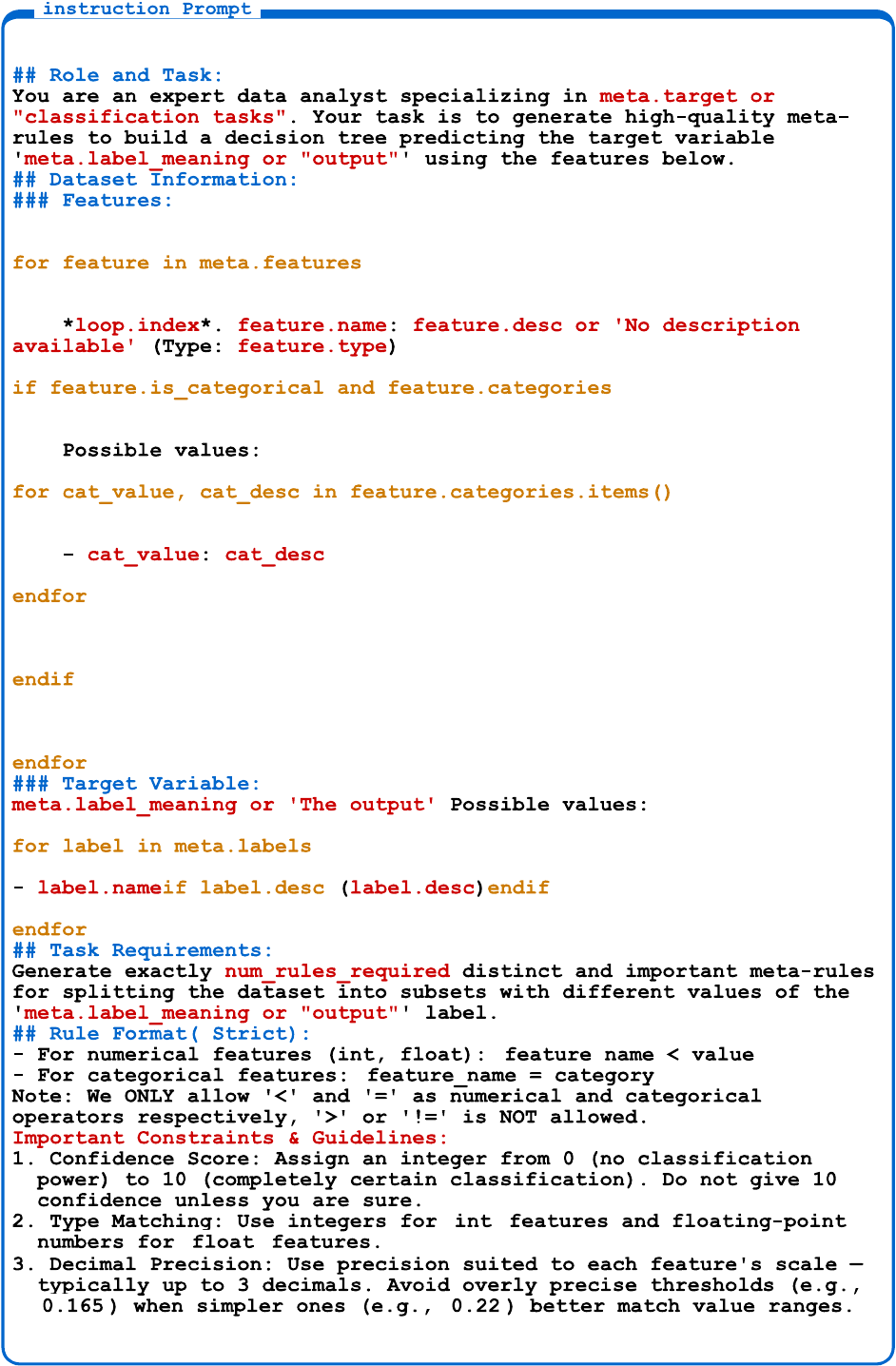}
    \caption{Meta Rule Generation Prompting Template.}
    \label{fig:generation_prompt}
\end{figure*}
\begin{figure*}[t]
    \centering
    \includegraphics[width=0.8\linewidth]{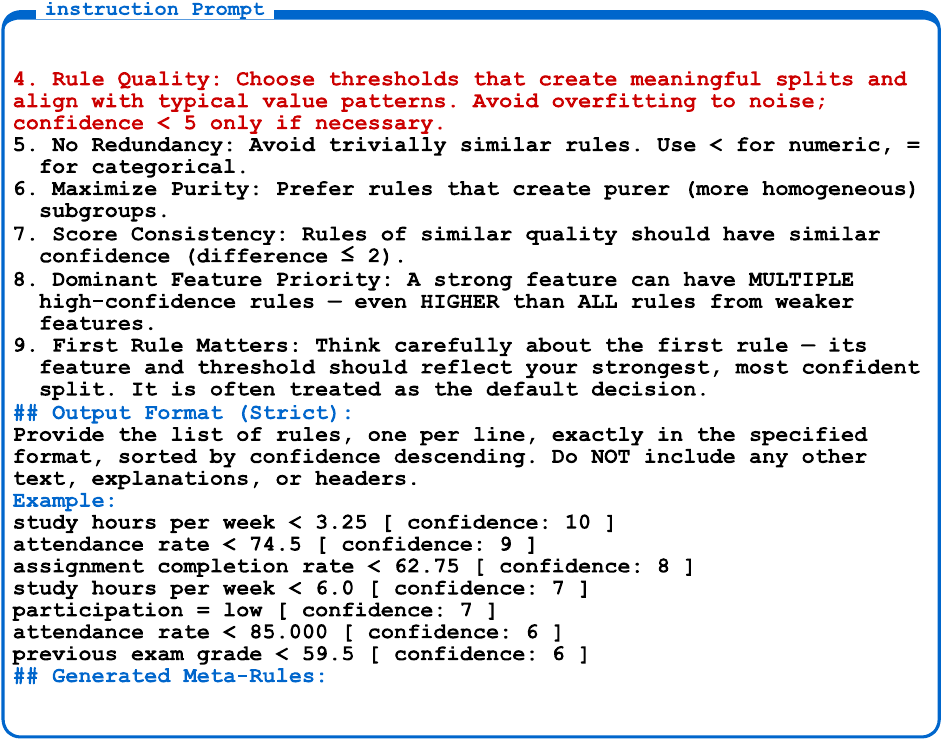}
    \caption{Meta Rule Generation Prompting Template (continued).}
    \label{fig:generation_prompt_cont}
\end{figure*}

\enlargethispage{100cm} % 扩展1cm高度

\begin{figure*}[t]
    \centering
    \includegraphics[width=0.8\linewidth]{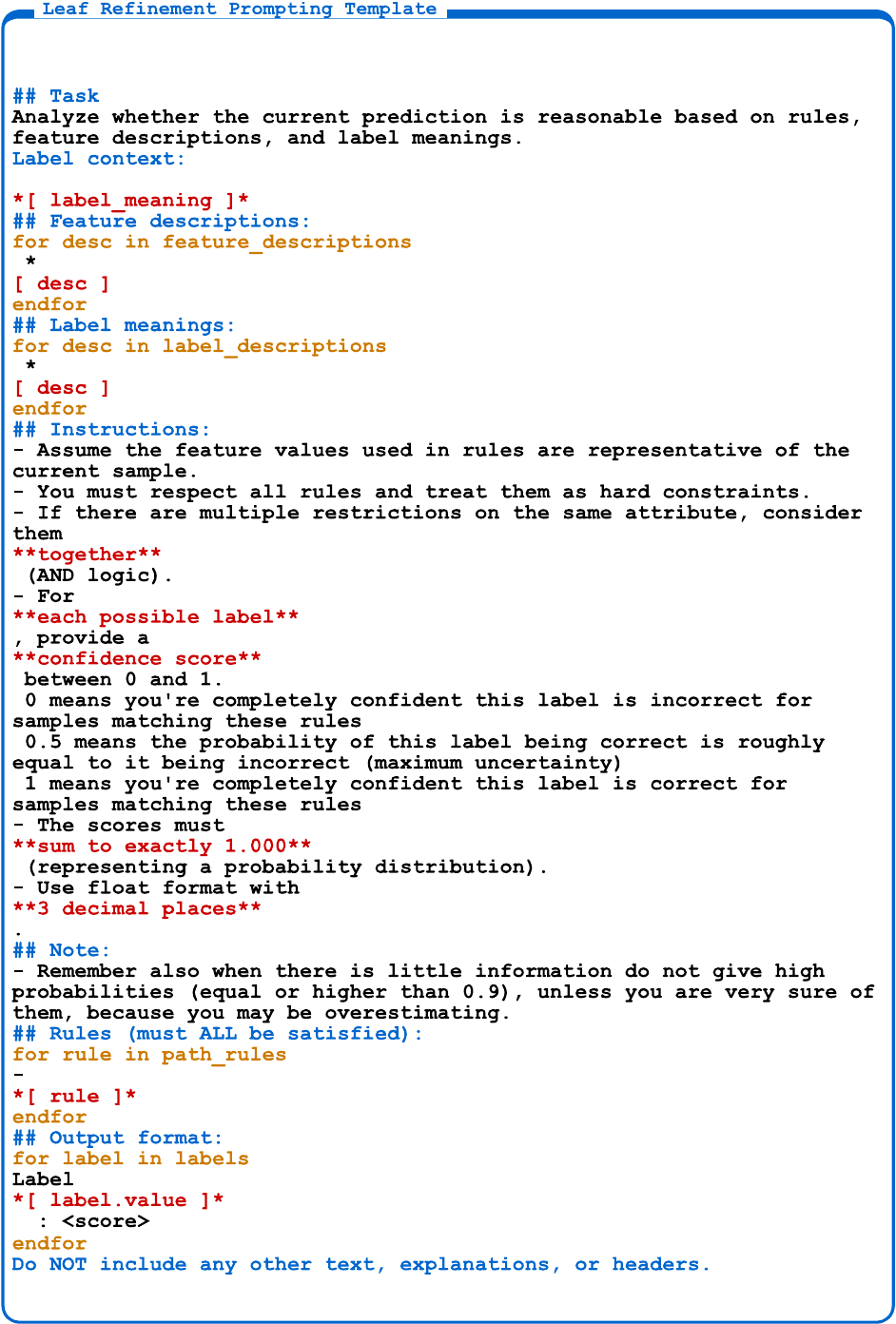}
    \caption{Leaf Refinement Prompting Template.}
    \label{fig:refinement_prompt}
\end{figure*}

\begin{figure*}[t]
    \centering
    \includegraphics[width=0.8\linewidth]{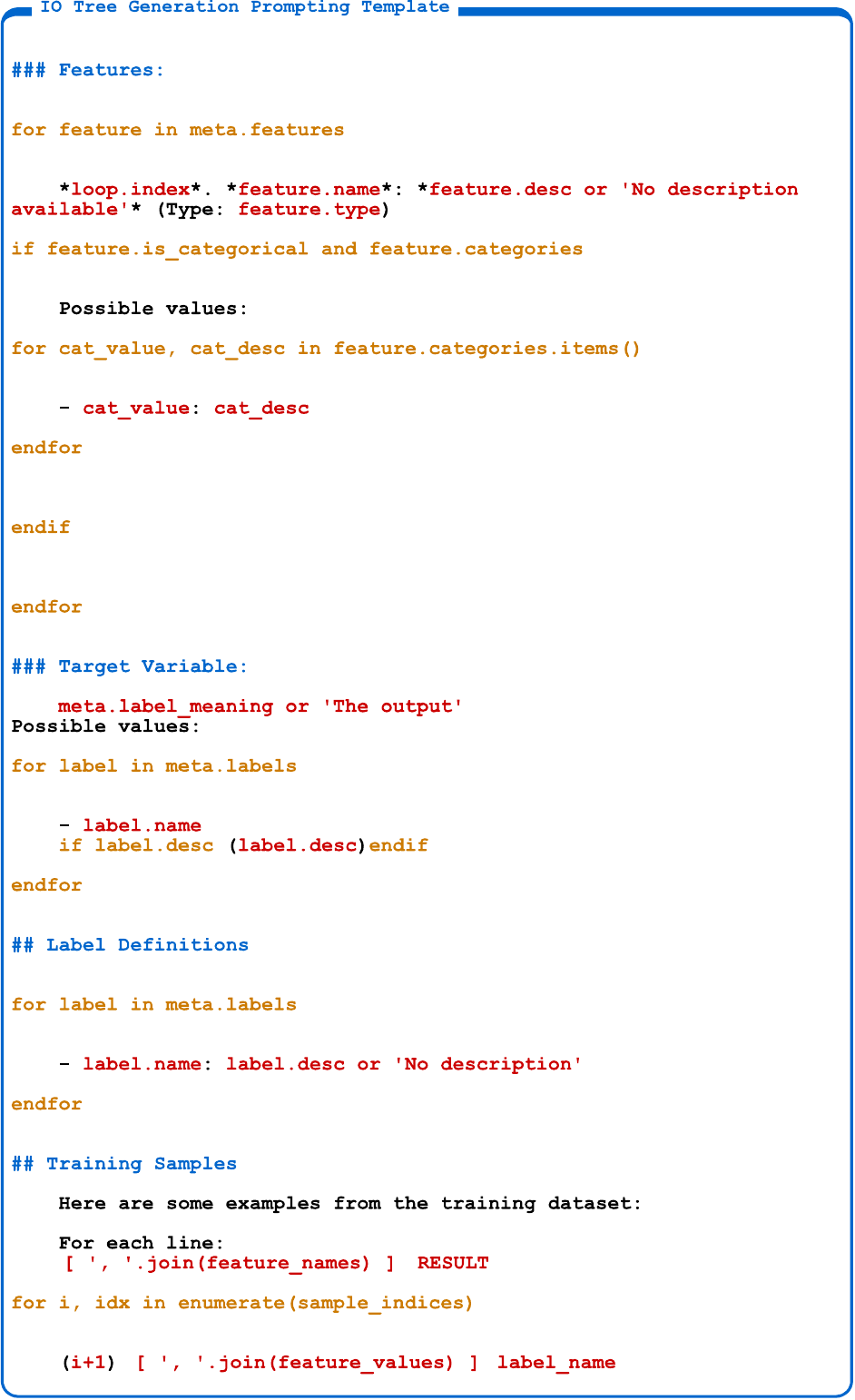}
    \caption{IO Tree Generation Prompting Template.}
    \label{fig:io_prompt}
\end{figure*}

\begin{figure*}[t]
    \centering
    \includegraphics[width=0.8\linewidth]{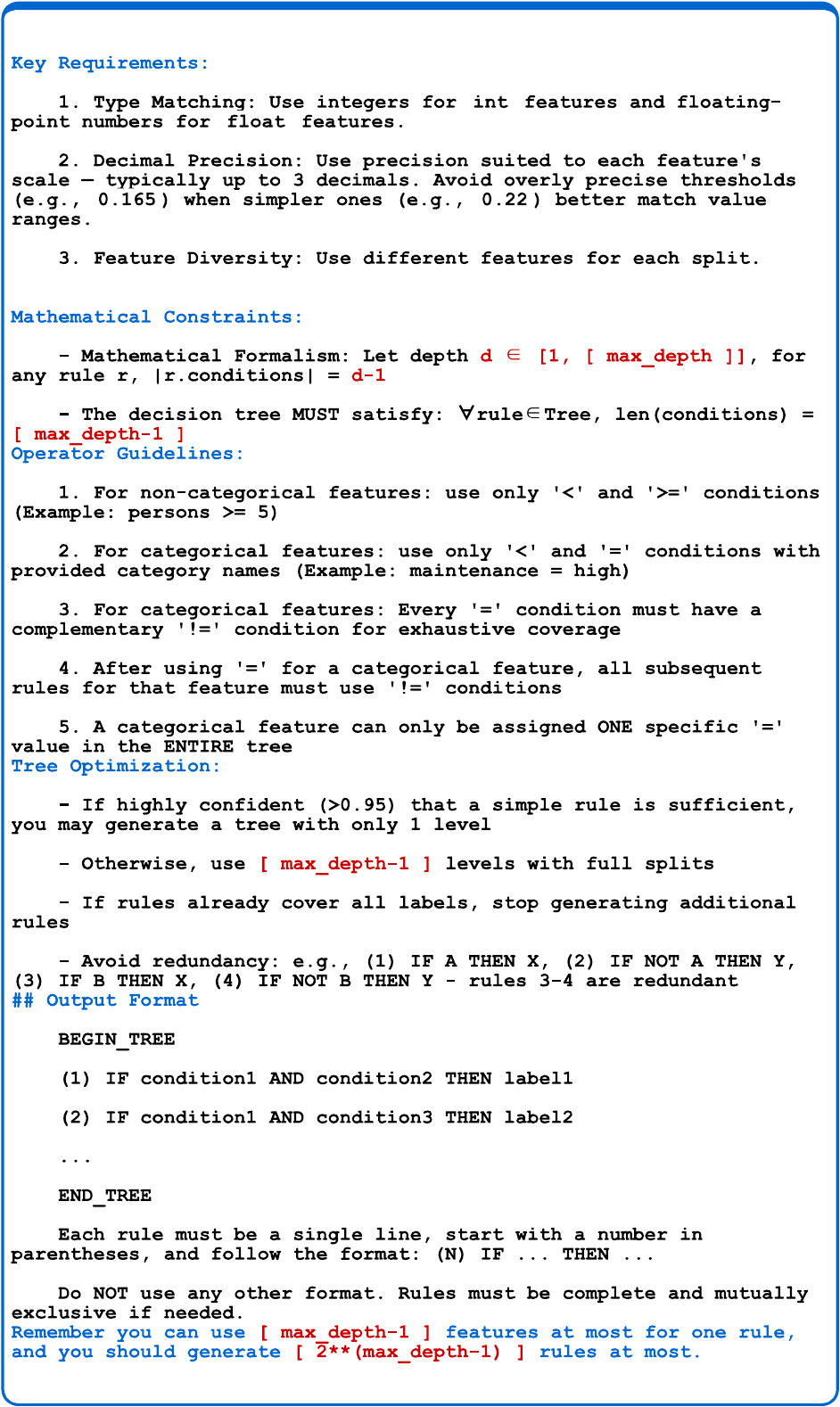}
    \caption{IO Tree Generation Prompting Template (continued).}
    \label{fig:io_prompt_cont}
\end{figure*}

\begin{figure*}[t]
    \centering
    \includegraphics[width=0.8\linewidth]{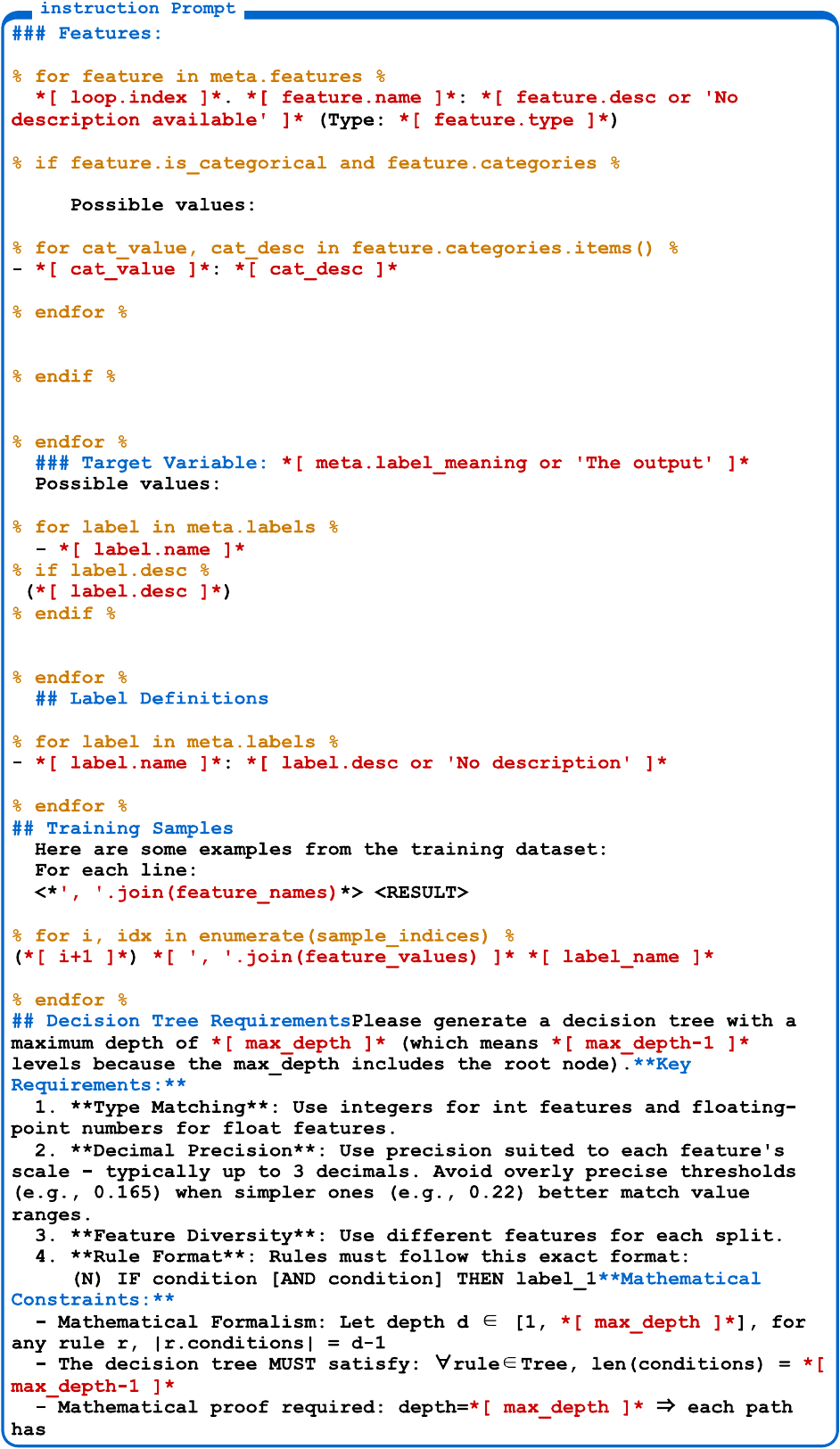}
    \caption{CoT Tree Generation Prompting Template.}
    \label{fig:cot_prompt}
\end{figure*}

\begin{figure*}[t]
    \centering
    \includegraphics[width=0.8\linewidth]{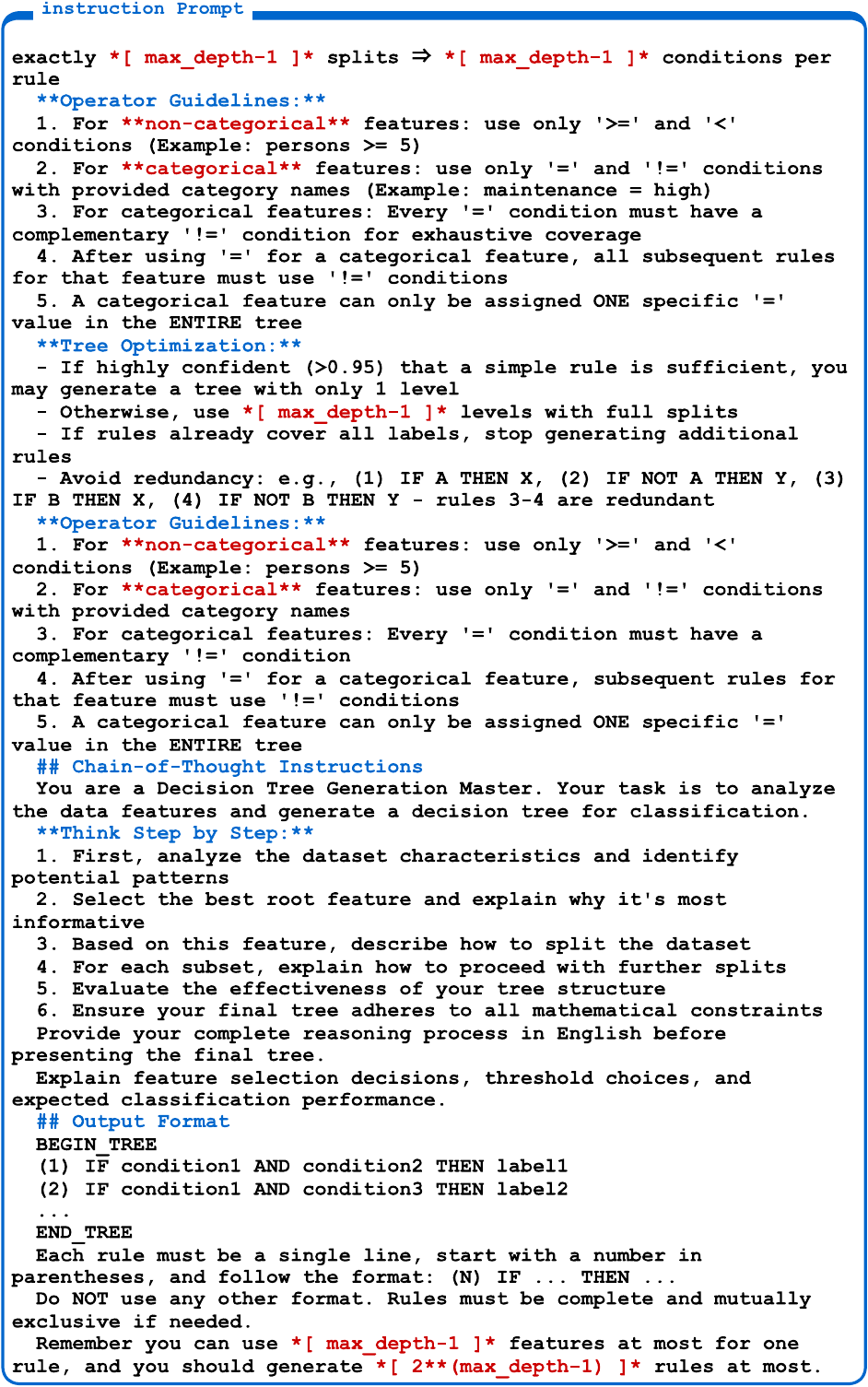}
    \caption{CoT Tree Generation Prompting Template (continued).}
    \label{fig:cot_prompt_cont}
\end{figure*}

\begin{figure*}[t]
    \centering
    \includegraphics[width=0.8\linewidth]{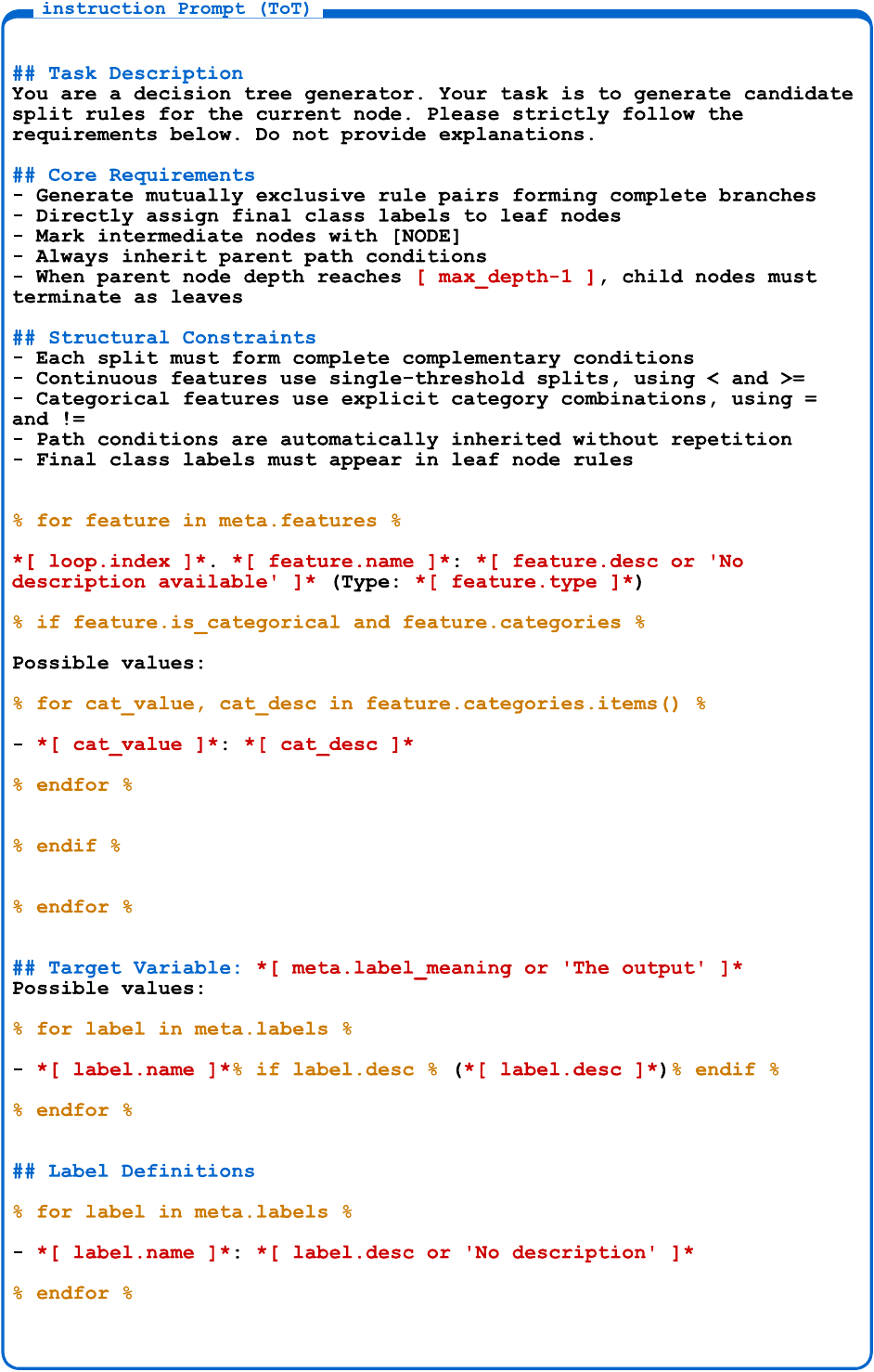}
    \caption{ToT Tree Generation Prompting Template.}
    \label{fig:tot_prompt}
\end{figure*}

\begin{figure*}[t]
    \centering
    \includegraphics[width=0.8\linewidth]{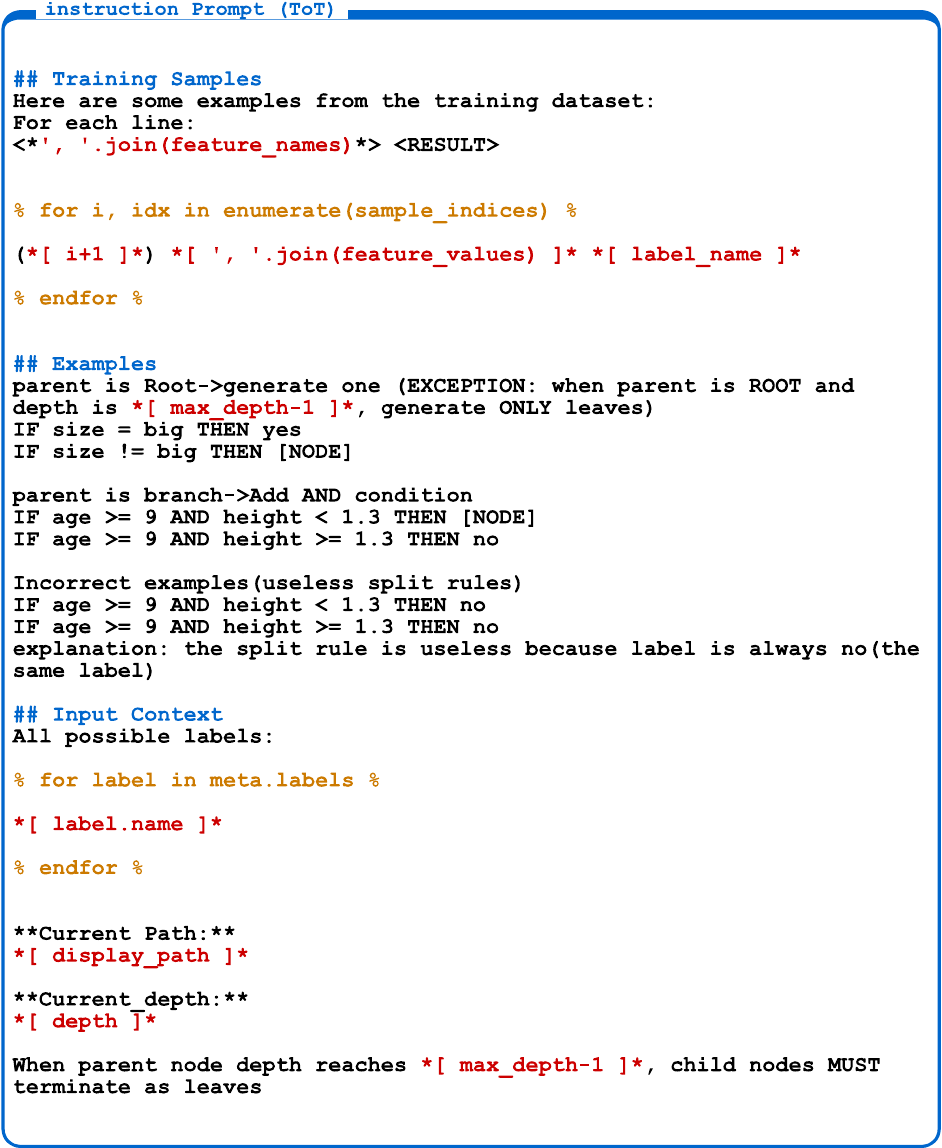}
    \caption{ToT Tree Generation Prompting Template (continued).}
    \label{fig:tot_prompt_cont}
\end{figure*}

\end{document}